\documentclass[11pt,reqno]{amsart}
\usepackage{amsmath}
\usepackage{amsthm}
\usepackage{graphicx}
\usepackage{amsfonts}
\usepackage{amssymb}
\usepackage{hyperref}
\usepackage{bm}
\usepackage{color}
\usepackage[margin=1.1in]{geometry}
\usepackage{enumerate}
\usepackage{mathrsfs}
\usepackage{mathtools}

\numberwithin{equation}{section}
\newtheorem{theorem}{Theorem}[section]
\newtheorem{lemma}[theorem]{Lemma}
\newtheorem{proposition}[theorem]{Proposition}
\newtheorem{corollary}[theorem]{Corollary}

\theoremstyle{definition}

\newtheorem{definition}[theorem]{Definition}
\newtheorem{remark}[theorem]{Remark}

\def\R{{\mathbb R}}
\def\N{{\mathbb N}}
\def\P{{\mathcal P}}

\makeindex
\begin{document}

\newcommand{\argmin}{\operatornamewithlimits{arg\,min}}
\newcommand{\argmax}{\operatornamewithlimits{arg\,max}}
\newcommand{\wgrad}{\nabla_\mathbb{W}}
\newcommand{\wpartial}{\partial_{\mathbb{W}}}
\newcommand{\W}{\mathbb{W}}
\newcommand{\wgradp}{\nabla_\mathbb{W}}
\newcommand{\h}{H} 

\title[Convergence of CAVI for log-concave measures]{Convergence of Coordinate Ascent Variational Inference for log-concave measures via optimal transport}
\author{Manuel Arnese}
\author{Daniel Lacker}
\address{Department of Industrial Engineering \& Operations Research, Columbia University}
\email{ma4339@columbia.edu, daniel.lacker@columbia.edu}
\thanks{M.A.\ is supported by the Unicredit Marco Fanno scholarship. D.L.\ is partially supported by the NSF CAREER award DMS-2045328.}
\date{}

\begin{abstract}
Mean field variational inference (VI) is the problem of finding the closest product (factorized) measure, in the sense of relative entropy, to a given high-dimensional probability measure $\rho$. The well known Coordinate Ascent Variational Inference (CAVI) algorithm aims to approximate this product measure by iteratively optimizing over one coordinate (factor) at a time, which can be done explicitly. Despite its popularity, the convergence of CAVI remains poorly understood.
In this paper, we prove the convergence of CAVI for log-concave densities $\rho$. If additionally $\log \rho$ has Lipschitz gradient, we find a linear rate of convergence, and if also $\rho$ is strongly log-concave, we find an exponential rate.
Our analysis starts from the observation that mean field VI, while notoriously non-convex in the usual sense, is in fact displacement convex in the sense of optimal transport when $\rho$ is log-concave.
This allows us to adapt techniques from the optimization literature on coordinate descent algorithms in Euclidean space.
\end{abstract}

\maketitle

\section{Introduction}

Variational inference (VI) is a method to approximate probability measures that has found great success in applications to Bayesian statistics, rivaling Markov Chain Monte Carlo methods. The main idea of variational inference is to fix a family of probability measures $\mathcal{F}$ and find the element of this family which is closest in relative entropy to the measure $\rho$ that we wish to approximate. That is, we approximate $\rho$ by solving the optimization problem
\[\inf_{\mu \in \mathcal{F}}\,\h(\mu\,\|\,\rho),\]
with $\h$ being the relative entropy, or Kullback-Leibler divergence.  
Different families $\mathcal{F}$ present different advantages and disadvantages, and likewise require different algorithms. Aside from the family of Gaussian measures, one of the most common choices is the mean field family, where $\mathcal{F}$ is the class of product probability measures. This gives rise to what is known in the literature as Mean Field VI (MFVI), for a measure $\rho$ on $\R^k$:
\begin{equation}
\label{MFVI}
    \inf_{\mu^1,...,\mu^k  }
    \h(\mu^1\otimes...\otimes \mu^k \,\|\,\rho).
\end{equation} 
The algorithm most commonly used to solve problem \eqref{MFVI} is known as Coordinate Ascent Variational Inference (CAVI), which picks a coordinate $i$, selects the optimal marginal $\mu^i$ while the other marginals are kept fixed, and then switches coordinates. Part of the appeal of this algorithm is that the one-coordinate updates are explicit; see \eqref{def:update-explicit} below. Two versions of CAVI are widely used: a sequential implementation, where coordinates are updated one by one, and a parallel implementation where coordinates are updated at the same time. See \cite{JordanVI} and \cite{BleiVI} for detailed introductions to VI, with the latter featuring a thorough discussion of CAVI.

Despite its popularity, theoretical convergence guarantees for the CAVI algorithm are scarce. After all, the optimization problem \eqref{MFVI} is famously non-convex \cite[Section 5.4]{JordanVI}.
To the best of our knowledge, the only paper proving convergence rates for a general target measure  $\rho$ is the recent \cite{bhattacharya2023CAVI}, which proves exponential convergence in the two-block case under an intriguing new condition of low ``generalized correlation" for $\rho$; we refer to Section \ref{lit review} for further discussion.
Prior work has focused on the applications of CAVI to specific models, such as \cite{CAVIcommunity,WangVarBayes,BoweiVar,IsingVar}.

In this paper, we perform a convergence analysis of the sequential CAVI algorithm assuming that the target measure admits a log-concave density, $\rho=e^{-\psi}$.  We prove convergence to a minimizer under mild integrability assumptions on $\psi$. We obtain a linear convergence rate if also $\psi$ has Lipschitz gradient, and an exponential convergence rate if also $\psi$ is strongly convex.

The philosophy behind our approach is that MFVI \eqref{MFVI} is a \emph{geodesically convex} (or \emph{displacement convex}) problem in Wasserstein space, in the sense of McCann \cite{mccann1997convexity}, when $\rho$ is log-concave. See Section \ref{se:geometry} for a detailed discussion.
Our analysis is then based on viewing CAVI as an equivalent in Wasserstein space of the (Block) Coordinate Descent (BCD) algorithm from the convex optimization literature, described in \cite[page 160]{bertsekas1997} or \cite[Chapter 14]{beck2017first}. The analysis of BCD uses classical tools from convex optimization, and in particular the gradient of the objective function plays a central role. To take the place of Euclidean gradients, we use Wasserstein gradients in the sense of Otto calculus  \cite{OttoGeometry,Ambrosio2008-dt}, with some adaptations to the sub-manifold of product measures within  Wasserstein space.
In the end, our convergence results closely parallel those known for the BCD algorithm applied to the convex Euclidean function $\psi$.

Our work fits into the booming literature on  connections between optimization and sampling, which leans heavily on the geometry of Wasserstein space and the theory of its gradient flows, and we refer to the monograph-in-progress \cite{chewi2023log} for a beautiful introduction to this connection. The
assumption of log-concavity of $\rho$ is not an innocuous one, and we do not claim that it covers all
cases of interest. However, log-concavity plays the same role for entropy-minimization problems
on Wasserstein space that convexity plays for standard Euclidean optimization. Of course, nonconvex Euclidean optimization problems frequently arise in applications, but few would question
the significance of convex optimization theory. The log-concave family is sufficiently broad to
cover applications of interest, one of which we discuss in Section \ref{se:regression}, and more generally it serves
as a natural benchmark or best case scenario. See \cite{saumard2014log} for a survey of log-concave probability
measures, which includes examples and applications in statistics and machine learning.

\subsection{Setting and main results}

The target probability measure $\rho$ on $\R^k$ is assumed to admit a density function proportional to $e^{-\psi(x)}$, where $\psi : \R^k\to\R$ is a given convex function. We routinely identify a measure with its density when it exists. We will sometimes assume that $\psi$ is $\lambda$-strongly convex, which means that $x\mapsto \psi(x)-\frac{\lambda}{2}|x|^2$ is convex.
Recall the definition of relative entropy
\[\h(\mu\,\|\,\nu) := \begin{cases}
    \int \frac{d\mu}{d\nu} \log\frac{d\mu}{d\nu} \,d\nu &\text{if }\mu\ll\nu,\\
    +\infty &\text{otherwise.}
\end{cases}\]
We partition our variables $1,\ldots,k$ into $d \leq k$ blocks, and assume that each block has dimension $k_i$ with $\sum_i k_i = k$.
A generic $x \in \R^k$ is written $x=(x^1,\ldots,x^d)$, with $x^i \in \R^{k_i}$.

Let $\P(\R^k)$ denote the space of probability measures on $\R^k$, and let $\P^{\otimes d}(\R^k)$ denote the subset of product measures, under which the blocks are independent:
\begin{equation}
\P^{\otimes d}(\R^k) = \{\mu^1 \otimes \cdots \otimes \mu^d : \mu^i \in \P(\R^{k_i}), \ i=1,\ldots,d\} \subset \P(\R^k). \label{def:productmeasures}
\end{equation}
We will routinely identify $\P^{\otimes d}(\R^k)$ with $\bigtimes_{i=1}^d \P(\R^{k_i})$ in the natural way, by identifying a product measure $\mu^1 \otimes \cdots \otimes \mu^d$ with the vector $(\mu^1,\ldots, \mu^d)$ of marginal measures. With this notation in mind, we may write the MFVI problem as
\begin{equation}
\inf_{\mu \in \P^{\otimes d}(\R^k)} \h(\mu\,\|\,\rho). \label{MFVI-formal}
\end{equation}

The sequential CAVI algorithm is defined as follows, as discussed in \cite[Section 2.4]{BleiVI}. 
Fix an initial product measure $\mu_0=\bigotimes_{i=1}^d \mu_0^i$.
Denote by $\mu_n=\bigotimes_{i=1}^d \mu_n^i$ the CAVI iterates, defined explicitly by the probability densities
\begin{equation}
   \mu_{n+1}^i(x^i)\propto \exp\bigg(-\int_{\R^{k-k_i}}\psi(x^1,\ldots,x^n)\,\bigotimes_{j < i} \mu^j_{n+1}(dx^j) \bigotimes_{j > i} \mu^j_{n}(dx^j)\bigg). \label{def:update-explicit}
\end{equation}
We will see in Lemma \ref{form of the iterates lemma} that this is well-defined under the assumptions of Theorem \ref{main theorem}, in the sense that the right-hand side of \eqref{def:update-explicit} is $dx^i$-integrable, i.e., the hidden constant of proportionality is finite (and non-zero).
Moreover, $\mu_{n+1}^i$ is the unique optimizer
\begin{equation}
\label{Cavi coordinate}
    \mu^{i}_{n+1} = \argmin_{\nu \in \P(\R^{k_i})} H(\mu_{n+1}^1 \otimes  \cdots \otimes \mu_{n+1}^{i-1}\otimes\nu \otimes \mu_n^{i+1}\otimes \cdots\otimes \mu_n^d\,\|\,\rho),
\end{equation}
whenever the value of the minimum on the right-hand side is finite, which holds automatically once $n \ge 1$.
It follows from \eqref{Cavi coordinate} that $\h(\mu_n\,\|\,\rho)$ is non-increasing in $n$. 

In what follows, let $\W_2$ denote the Wasserstein distance (defined in \eqref{def:W2}). 

\begin{theorem} \label{main theorem}
Let $\psi : \R^k \to \R$ be a convex function such that $\rho(x) \propto e^{-\psi(x)}$ defines a probability density, and assume there exist finite constants $c > 0$ and $p \ge 2$ such that
\begin{equation}
|\psi(x)| \leq c(1+|x|^p), \qquad |\nabla\psi(x)|\leq c(1+|x|^p), \qquad \text{for almost every } x\in \R^k, \label{asmp:growth-psi}
\end{equation}
where $\nabla \psi$ is the weak gradient.
Let $\mu_0 \in \P^{\otimes d}(\R^{k_i})$ have finite $p$moment, and define the CAVI iterates $\mu_n=\bigotimes_{i=1}^d \mu_n^i$ as in \eqref{def:update-explicit}. Then $H(\mu_1\,\|\,\rho) < \infty$, and the following hold:
    \begin{enumerate} 
\item The sequence $(\mu_n)$ is tight, and every weak limit point is a minimizer of \eqref{MFVI-formal}. 
\item If $\psi$ is strictly convex, then \eqref{MFVI-formal} admits a unique minimizer $\mu_*$, and $\mu_n \to \mu_*$ weakly.
        \item If $\psi$ is differentiable and $\nabla\psi$ is $L$-Lipschitz for $L > 0$, and if $\mu_*$ is a minimizer of  \eqref{MFVI-formal}, then  
    \[H(\mu_{n}\,\|\,\rho)-H(\mu_*\,\|\,\rho) \leq \left(2+H(\mu_1\,\|\,\rho)-H(\mu_*\,\|\,\rho)+\frac{1}{R\sqrt{Ld}}\right)\cdot\frac{2R^2Ld}{n},\]
    where $R := \sup_{n \in \N}\W_2(\mu_n,\mu_*) < \infty$.

    \item If $\psi$ is $\lambda$-strongly convex with $L$-Lipschitz gradient, for $L \ge \lambda > 0$, then \eqref{MFVI-formal} admits a unique minimizer $\mu_*$, and
    \[\frac{\lambda}{2} \W_2^2(\mu_n,\mu_*)\leq H(\mu_{n}\,\|\,\rho)-H(\mu_*\,\|\,\rho)\leq \left(1-\frac{\lambda^2}{L^2 d+\lambda^2}\right)^{n-1}\big(H(\mu_{1}\,\|\,\rho)-H(\mu_*\,\|\,\rho)\big).\]
    \end{enumerate}
\end{theorem}

The main assumptions made in Theorem \ref{main theorem} regard the convexity and regularity of $\psi$, which are reflected in convexity and regularity of the objective functional $H(\mu_*\,\|\,\rho)$. Indeed, $\psi$ is (strongly) convex if and only if $H(\mu_*\,\|\,\rho)$ is (strongly) geodesically convex \cite[Section 9.4.1]{Ambrosio2008-dt}. 
The rate itself mirrors common results in convex optimization: linear convergence for convex objectives and exponential for strongly convex objectives. The constant $R$ is the diameter of the sequence of iterates; while this quantity is usually assumed to be finite in the convex optimization literature, we have explicit bounds on its size, discussed in detail in Section \ref{se:Restimates}.

We prove part (1) in two stages: First, even without convexity of $\psi$, we show that limit points of $(\mu_n)$ are \emph{stationary points} (coordinatewise minimizers) of $\h(\mu^1 \otimes \cdots \otimes \mu^d\,|\,\rho)$. This result (stated in Proposition \ref{qualitative converegence}), unlike our others, does not involve any Wasserstein geometry and might be known already, though we were unable to find a reference. Second, we deduce (1) by showing that stationary points are minimizers of \eqref{MFVI-formal} when $\psi$ is convex. Part (2) then follows from uniqueness of \eqref{MFVI-formal} under the strict log-concavity assumption, which is known from \cite{MeanFieldApprox}. The bulk of our efforts are in proving the quantitative parts (3) and (4), and we discuss the ideas in the next section.

\begin{remark} \label{re:convergencemode}
The mode of convergence in Theorem \ref{main theorem} can be strengthened in two directions:
\begin{itemize}
\item In Lemma \ref{form of the iterates lemma}(1) below we show that  the $q$th moment of $\mu_n$ is bounded in $n \ge 1$, for any $q \in [1,\infty)$, and thus the weak convergence in Theorem \ref{main theorem} can be upgraded to the $q$-Wasserstein convergence for any $q$.
\item In Proposition \ref{pr:normcompact} in the appendix we show that the sequence of densities $(\mu_n)_{n \ge 1}$ is  precompact in the norm topology of $L^q(\R^k)$, for any $q \in [1,\infty)$, which implies that the (subsequential) convergence in Theorem \ref{main theorem} holds in this stronger topology as well.
\end{itemize}
\end{remark}

\subsection{Wasserstein Geometry of MFVI} \label{se:geometry}

As alluded to above, the MFVI problem \eqref{MFVI-formal} fails to be a convex optimization problem, when the space of probability measures inherits the usual vectorial structure of signed measures. Indeed, one could view MFVI as the minimization of the convex function $\h(\cdot\,\|\,\rho)$ over the non-convex subset $\P^{\otimes d} (\R^k)$  of $\P(\R^k)$, or as the minimization over the convex set $\bigtimes_{i=1}^d \P(\R^{k_i})$  of the non-convex $d$-argument function $\h(\mu^1\otimes\cdots\otimes \mu^d \,\|\,\rho)$.

The starting point of our analysis is that MFVI is a \emph{geodesically convex} optimization problem on Wasserstein space.
 Let $\P_2(\R^k)$ denote the space of probability measures on $\R^k$ with finite second moments, equipped with the Wasserstein distance defined below in \eqref{def:W2}. Let $\P_2^{\otimes d}(\R^k) = \P_2(\R^k) \cap \P^{\otimes d}(\R^k)$.
Under the assumptions of Theorem \ref{main theorem}, it turns out that  optimizers of \eqref{MFVI-formal} must have finite second moment, and so
\begin{equation}
\inf_{\mu \in \P^{\otimes d}(\R^k)}\h(\mu\,\|\,\rho) = \inf_{\mu \in \P_2^{\otimes d}(\R^k)}\h(\mu\,\|\,\rho). \label{MFVI-P2}
\end{equation}
The constraint set $\P_2^{\otimes d}(\R^k)$ is geodesically convex, as observed in \cite[Proposition 3.2]{lacker2023IndependentProjection}. That is, let $\nu_0,\nu_1 \in \P_2^{\otimes,d}(\R^k)$ be absolutely continuous, let $T_{\nu_0 \to \nu_1}$ denote the optimal transport (Brenier) map from $\nu_0$ to $\nu_1$, and let $\nu_t$ denote the pushfoward of $\nu_0$ by the map $x \mapsto (1-t)x + tT_{\nu_0\to\nu_1}(x)$, i.e., the \emph{displacement interpolation}. 
Then $\nu_t$ remains in $\P_2^{\otimes d}(\R^k)$ for every $t\in (0,1)$. 
Intuitively, the optimal transport between two product measures is obtained by optimally transporting between the corresponding marginals, i.e., $T_{\nu_0\to\nu_1}(x^1,\ldots,x^n) = (T_{\nu_0^i\to\nu_1^i}(x^i))_{i=1,\ldots,n}$.

A remarkable and well known fact from optimal transport \cite[Section 9.4.1]{Ambrosio2008-dt} is that a probability measure $\rho$ is log-concave if and only if the relative entropy $\h(\cdot\,\|\,\rho)$ is a geodesically convex function on $\P_2(\R^k)$. That is, if $\nu_t$ is the displacement interpolation mentioned in the previous paragraph, then $\h(\nu_t\,\|\,\rho)$ is a convex function of $t \in [0,1]$.
Hence, \eqref{MFVI-P2} is a convex optimization problem when viewed through the lens of Wasserstein geometry, as it is the minimization of a geodesically convex functional over a geodesically convex set.

Embracing this perspective, our proofs are based on a combination of techniques from optimal transport and convex optimization.  We will make extensive use of the notion of Wasserstein gradient, treated rigorously in \cite{Ambrosio2008-dt} as the minimal norm subgradient.
Despite its abstract definition, the Wasserstein gradient takes a simple tractable form for many functionals of interest. In particular, for nice enough $\mu$, the following formula is well known:
\begin{equation}
\nabla_\W \h(\mu\,\|\,\rho)(x) = \nabla \log \frac{\mu(x)}{\rho(x)} = \nabla \log \mu(x) + \nabla \psi(x). \label{def:Wgrad-H}
\end{equation}
This expression will appear in several convexity inequalities that are inspired by their Euclidean analogues.
As a first simplest example, the inequality $f(x) \ge f(y) + \langle \nabla f(y),x-y\rangle$ for a convex function on  Euclidean space admits the following analogue for a geodesically convex function $\Phi$ on $\P_2(\R^k)$:
\begin{equation*}
\Phi(\mu)\geq \Phi(\nu)+\int_{\R^k} \left\langle\wgrad \Phi(\mu,x),T_{\nu \to \mu}(x)-x\right\rangle\,\nu(dx),
\end{equation*}
where $T_{\nu \to \mu}$ is the optimal transport map from $\nu$ to $\mu$ and $\wgrad \Phi(\mu,\cdot)$ is the Wasserstein gradient of $\Phi$ at $\mu$. This  inequality has found application in the analysis of Langevin Monte Carlo  using  methods of convex optimization \cite{durmus2019analysis} (see also \cite[Section 4.3]{chewi2023log}), similar in spirit to our work.

On the other hand, it is typically more difficult to implement Wasserstein space analogues of \emph{descent lemmas}, to ensure sufficient decrease of the objective at each iterate. 
Descent lemmas for gradient-descent-type algorithms normally require that the convex objective function $f$ on Euclidean space has $L$-Lipschitz gradient, which is equivalent to the following useful inequality \cite[Theorem 5.8]{beck2017first}:
\begin{align*}
f(y) &\ge f(x) + \left \langle \nabla f(x), y-x\right\rangle+\frac{1}{2L}|\nabla f(y)- \nabla f(x)|^2.
\end{align*}
To develop a descent lemma in Wasserstein space, a natural analogue of the ``$L$-Lipschitz gradient" assumption might be that $\wgrad \h(\mu\,\|\,\rho)(x)$ is Lipschitz jointly in $(\mu,x)$. But this can never hold \emph{globally}, simply because $\h(\mu\,\|\,\rho)$ blows up for certain $\mu$.
Fortunately, the explicit update rule \eqref{def:update-explicit} imparts a rigid structure on the iterates $\mu_n^i$, which we exploit extensively to resolve these smoothness issues.

\subsection{CAVI as BCD in Wasserstein Space}

The BCD algorithm has emerged as a powerful tool for \emph{composite} optimization problems of the form
\begin{equation}
\label{BCD problem}
    \min_{x \in \R^k} F(x) := f(x)+\sum_{i=1}^d g^i(x^i),
\end{equation}
for functions $f:\R^k \to \R$ and $g^i:\R^{k_i}\to \R$, which
are common in  statistics and machine learning communities.
Typically, the function $f$ is well-behaved, say (uniformly) convex with Lipschitz gradient, whereas each of the functions $g^i$ may be non-smooth and even discontinuous but, conveniently, acts only on a single variable.
For example, the LASSO is defined by \eqref{BCD problem} with $f(x)$ quadratic and $g^i(x^i)=|x^i|$. The iterates of the BCD algorithm are defined as:
\[x_{n+1}^i \in \argmin_{y \in \R^{k_i}}F\big(x_{n+1}^1,\ldots,x_{n+1}^{i-1},y,x_n^{i+1},\ldots,x^d_n\big).\]
Clearly this resembles \eqref{Cavi coordinate}.
The behavior of BCD and its variants are by now well understood, thanks to work done in the last twenty years, starting from \cite{QualitativeConvergenceTseng2001} (and earlier references therein) for qualitative convergence and \cite{Sahaconvergence} for convergence rates.

In fact, the MFVI problem can be written in precisely the composite form \eqref{BCD problem}.
Define the (negative) differential entropy 
\begin{equation}
h(\mu):= \begin{cases}
    \int_{\R^k} \mu(x)\log\mu(x)\,dx  &\text{if }\mu \text{ admits a density with } \mu\log \mu \in L^1(\R^k) \\
    +\infty &\text{otherwise.}
\end{cases} \label{def:differentialentropy}
\end{equation}
For a product measure $\mu= \mu^1 \otimes \cdots \otimes \mu^d$, the entropy tensorizes,  $h(\mu) = h(\mu^1)+\cdots+h(\mu^d)$, and thus
\begin{equation}
\h(\mu\,\|\,\rho)=\int_{\R^{k}} \psi \,d\mu + \sum_{i=1}^d h(\mu^i). \label{eq:Hcomposite}
\end{equation}
It is well known that $h$ is geodesically convex and lower semicontinuous in Wasserstein space, but it is not differentiable in the Wasserstein sense, except at sufficiently nice measures $\mu$. On the other hand, the linear part $\int \psi\,d\mu$ inherits  regularity from $\psi$ itself; for instance, it is geodesically (strongly) convex if and only if $\psi$ is (strongly) convex. 
Moreover, the  Wasserstein gradient of $\mu \mapsto \int \psi\,d\mu$ is just $\nabla\psi(x)$, and thus the former shares Lipschitz assumptions with the latter.
This puts us squarely in the setting of \eqref{BCD problem}, and it is thus natural that CAVI would behave well in analogy with the Euclidean setting.
We are certainly not the first to view relative entropy minimization as a composite optimization problem, a perspective which is thoughtfully exploited in \cite{wibisono2018sampling} for the  analysis and design of sampling algorithms.

The analogy with BCD suggests several possible generalizations of our work. First, it is conceivable that the space of product measures could be replaced by another geodesically convex subset thereof; optimization problems of this form show up in practical implementations of MFVI, in which one restricts from the non-parametric mean field family to a more tractable parametric sub-family.
Second, our methods could likely yield similar convergence results for more general coordinate descent problems on the space of probability measures, in which the entropy functional $H(\cdot\,\|\,\rho)$ is replaced by a composite functional of the form $\mathcal{F}(\mu^1 \otimes \cdots \otimes \mu^d) + \sum_i \mathcal{G}(\mu^i)$, where $\mathcal{F}$ and $\mathcal{G}$ are geodesically convex and $\mathcal{F}$ is smooth.  
Lastly, it should be possible to replace the underlying Euclidean spaces by Riemannian manifolds, where the geometry of Wasserstein space is essentially equally well understood.

\subsection{Applications to Bayesian Linear Regression}\label{se:regression}

We showcase the applicability of our results in a common setting of Gaussian linear regression. For simplicity, consider $d=k$ blocks of dimension $k_i=1$. In this model we observe data $(y_i,X_i)_{i=1}^m$ with $y_i \in \R$ while $X_i \in \R^k$. Let $y=[y_1,\ldots,y_m]^\top$ and $X= [X_1,\ldots,X_m]^\top$ and suppose we wish to estimate the  model
\[y=X\beta + \varepsilon,\]
where $\varepsilon\sim \mathcal{N}(0,\sigma^2 I)$ for some $\sigma > 0$. The vector $\beta=[\beta_1,\ldots,\beta_k]^\top$ has i.i.d.\ components with marginal $\beta_i \sim\pi\propto e^{-\phi}$, where $\phi : \R \to \R$ has bounded second derivative, $A\ge \phi''(x)\geq a$. Assume also that $bI\leq X^\top X \leq BI$ in semidefinite order, for some constants $b \le B$. 

The posterior conditional on our data is then given by the density
\[\rho(\beta) =\pi(\beta\mid y,X)\propto \exp\left\{-\sum_{i=1}^k \phi(\beta_i)-\frac{1}{2\sigma^2}|y-X\beta|^2\right\}.\]
The posterior is often approximated by using MFVI, as discussed at length in \cite{BleiVI}. Let $\mu_*$ be the solution of  \eqref{MFVI}, and let $\mu_n$ be the iterates of the resulting CAVI algorithm, initialized from some product measure having finite second moment. 

\begin{proposition} 
In the setting above, let $\lambda = a+b\sigma^{-2}$ and $L=A + B\sigma^{-2}$. If $\lambda \ge 0$, then Theorem \ref{main theorem}(3) applies with this choice of $L$. 
If $\lambda > 0$, then Theorem \ref{main theorem}(4) applies with this choice of $(\lambda,L)$.
\end{proposition}

We omit the proof, which is a simple application of Theorem \ref{main theorem} (with $p=2$), after computing the log-Hessian of $\rho=\pi(\cdot\mid y,X)$:
\begin{equation*}
-\nabla^2_\beta \log \rho(\beta) = \begin{pmatrix}
\phi''(\beta_1) & \ & \ \\
\ &  \ddots & \ \\
\ & \ & \phi''(\beta_k) 
\end{pmatrix} + \frac{1}{\sigma^2} X^\top X.
\end{equation*}
The eigenvalues are contained in $[\lambda,L]$. Thus
$\psi(\beta)=-\log\rho(\beta)$ is convex when $\lambda \ge 0$. In particular, it is $\lambda$-strongly convex when $\lambda > 0$. Moreover, if $\lambda \ge 0$, we see that $\nabla \psi$ is Lipschitz with constant $L$. 
Note that $X^\top X$ is positive semidefinite, so we may always take $b \ge 0$, but we impose no sign assumptions on $a$ or $A$ and just need $\lambda = a+b\sigma^{-2} \ge 0$. This implies that non-convexity of $\phi$ ($a < 0$) can be offset by a sufficiently positive realization of $X^\top X$ ($b >0$). 

\subsection{The Gaussian case}
As another example, let us again take $d=k$ blocks of dimension $k_i=1$ for simplicity, and consider the case
\begin{equation*}
\psi(x)=\frac{1}{2}(x-m)^\top A(x-m), \quad x \in \R^k
\end{equation*}
for some given symmetric and positive definite matrix $A$ and $m \in \R^k$. This of course corresponds to $\rho\propto e^{-\psi}$ being a Gaussian measure, $\rho\sim \mathcal{N}(m,A^{-1})$. The unique MFVI minimizer $\mu_*$ is well known to be the Gaussian with the same mean $m$ and with inverse covariance matrix given by  zeroing out $A$ off the main diagonal. That is, $\mu_* \sim \mathcal{N}(m,(A \odot I)^{-1})$, where $\odot$ denotes the Hadamard (entrywise) product. See \cite[Example 5.3]{JordanVI} for additional discussion.

Theorem \ref{main theorem}(4) applies and yields exponential convergence to the MFVI minimizer, with $L$ and $\lambda$ therein taken to be the largest and smallest eigenvalues of $A$, respectively. In fact, we can improve upon Theorem \ref{main theorem}(4) in this case by removing the dependence on the dimension $d$ in the convergence rate. 

\begin{proposition}
\label{gaussian case}
In the setting above, let $\mu_n$ be the iterates of the CAVI algorithm, initialized from some product measure $\mu_0$ having finite second moment. 
Then, for each $n \ge 1$, $\mu_n \sim\mathcal{N}(m_n,(A \odot I)^{-1})$ for some $m_n \in \R^k$, and
\[
H(\mu_n\,\|\,\rho)-H(\mu_*\,\|\,\rho) = \psi(m_n) \leq \left(1-\frac{\lambda^2}{\lambda^2+64(L-\lambda)^2 \log^2(3)}\right)^{n-1}\psi(m_1).
\]
\end{proposition}

We suspect this result is known, but we give a proof in Appendix \ref{se:gaussian} for completeness, which is mainly an application of a result of \cite{StronglyConvexCaseLi2018} on BCD in Euclidean space. 
The main ideas are as follows.
In this setting, a single cycle of CAVI updates from any initialization $\mu_0$ results in $\mu_1$ being Gaussian with the same covariance matrix as the optimal $\mu_*$ but with a potentially different mean. Subsequent CAVI updates preserve the covariance matrix and simply update the mean, in exactly the same sequence as the (Euclidean space) BCD algorithm applied to the quadratic function $\psi$ (which happens to be equivalent to the Gauss-Seidel algorithm). The sharp recent results of \cite{StronglyConvexCaseLi2018} on BCD for quadratic functions yield the  dimension free rate.

\subsection{Related Literature}
\label{lit review}

While general theoretical guarantees on the convergence of CAVI have been unknown for a long time, there has been substantial progress in the past year.

We begin with the important recent paper \cite{bhattacharya2023CAVI}. They introduce a novel notion of \emph{generalized correlation}, and use calculus of variations techniques to prove that if the generalized correlation of the target measure is lower than $2$, then the two-block sequential and parallel CAVI converge exponentially fast to the minimizer. They prove a similar result for $d$-block parallel CAVI, with the caveat that the generalized correlation needs to be smaller than $2\slash\sqrt{d-1}$. Our methods are quite different, and our results are complementary. We exchange the assumption of low generalized correlation with convexity, and the scopes of these two conditions are not easy to compare. We obtain a worse rate in general, but a similar exponential convergence rate in the strongly convex case. Our results have the advantage of extending seamlessly to an arbitrary number of blocks of arbitrary dimension.
Another important difference is that \cite{bhattacharya2023CAVI}  focuses on parallel CAVI, more than sequential. These two versions behave very differently in high dimension. The Gaussian case provides an excellent example: Proposition \ref{gaussian case} shows that sequential CAVI for Gaussians always converges with a \textit{dimension free rate}, while the discussion following \cite[Theorem 6]{bhattacharya2023CAVI} provides an example of \textit{non-convergence} of parallel CAVI with a Gaussian reference measure, if the dimension (and the number of blocks) is at least $3$.

Another notable recent work is \cite{Chewi23MFVI}, which develops fully implementable algorithms for optimization over finite dimensional, polyhedral, geodesically convex sets in Wasserstein space, with MFVI being their main application.
They prove convergence guarantees assuming, like us, that the reference measure is log-concave. Their algorithms are very different from CAVI, and their approach is different from ours but complementary: they develop a finite dimensional approximation of \eqref{MFVI}, which allows them to apply finite-dimensional convex optimization techniques, whereas we focus on the infinite dimensional problem \eqref{MFVI}, lifting convex optimization techniques to Wasserstein space. Nonetheless, we share the focus on Wasserstein geometry and the emphasis on geodesic convexity.
This perspective of Wasserstein geometry is shared also by the recent papers \cite{lacker2023IndependentProjection} and \cite{lambert2022variational}, which study gradient flows in (submanifolds of) Wasserstein space for mean field and Gaussian variational inference, respectively, with the latter proving convergence rates for implementable discrete-time algorithms. See also \cite{yao2022mean} for a related algorithmic analysis of MFVI with latent variables, again from the perspective of Wasserstein gradient flows.

While our work is focused on the convergence of the CAVI algorithm toward the mean field optimizer $\mu_*$, we would be remiss not to mention the important though orthogonal problem of quantifying how close $\mu_*$ is to the target measure $\rho$. In the strongly log-concave setting of Theorem \ref{main theorem}(4), the recent work \cite{MeanFieldApprox} proves the sharp estimate $H(\mu_* \,\|\,\rho) \le \lambda^{-2} \sum_{i < j} \int |\partial_{ij}\psi|^2\,d\mu_*$ and explains how it leads to quantitative laws of large numbers, among other implications.

\subsection*{Outline of the paper}
The rest of the paper is organized as follows. In Section \ref{Wasserstein Calculus} we summarize the relevant result about Wasserstein calculus and the  geometry of $\P_2^{\otimes d}(\R^k)$. In Section \ref{CAVI ALGORITHM SECTION} we prove regularity and integrability properties for the CAVI iterates. After these technical preliminaries, we prove convergence of CAVI in Section \ref{qualitative section} for convex and sufficiently integrable $\psi$, i.e., parts (1) and (2) of Theorem \ref{main theorem}. We add an assumption of Lipschitz gradient in Section \ref{Lipschitz Gradient} and prove the linear convergence rate announced in Theorem \ref{main theorem}(3). We finally add a strong convexity assumption in Section \ref{strongly convex} and prove the exponential convergence of Theorem \ref{main theorem}(4).

\section{Convexity and calculus in Wasserstein space} \label{Wasserstein Calculus}

In this section we elaborate on some of the geometric and differential properties of $\P^{\otimes d}_2(\R^k)$ that will be key to our analysis. We   denote by $\P_{2,ac}(\R^k)$ the set of measures  with finite second moment that admit a density with respect to the Lebesgue measure, and similarly $\P_{2,ac}^{\otimes d}(\R^k) = \P^{\otimes d}_2(\R^k) \cap \P_{2,ac}(\R^k)$.
We will also write $\P_p(\R^k)$ for the set of probability measures on $\R^k$ with finite $p$th moment, and let $\P_p^{\otimes d}(\R^k) = \P^{\otimes d}(\R^k) \cap \P_p(\R^k)$.
Define the (quadratic) Wasserstein distance on $\P_{2}(\R^k)$ as usual by
\begin{equation}
\W_2(\mu,\nu)= \sqrt{ \inf_{\pi }\int_{\R^k\times \R^k} |x-y|^2\,\pi(dx,dy)}, \label{def:W2}
\end{equation}
where  the infimum is over couplings of $\mu$ and $\nu$, i.e., probability measures $\pi$ on $\R^k \times \R^k$ with first marginal $\mu$ and second marginal $\nu$.
If $\nu \in \P_{2,ac}(\R^k)$, then there exist a measurable function $T_{\nu \to \mu}:\R^k \to \R^k$, unique up to $\nu$-a.e.\ equality, such that the pushforward of $\nu$ by $T_{\nu \to \mu}$ is $\mu$  and 
\begin{equation*}
\W_2^2(\mu,\nu) = \int_{\R^k}|x-T_{\nu \to \mu}(x)|^2\,\nu(dx).
\end{equation*}
This is Brenier's theorem  \cite[Theorem 2.12]{villani2021topics}, and the map $T_{\nu \to \mu}$ is called the \emph{Brenier map} (or \emph{optimal transport map}).
In fact, the map $T_{\nu \to \mu}$ can alternatively be characterized as the unique gradient of a convex function which pushes forward $\nu$ to $\mu$.

\subsection{Geodesic convexity}
Given   $\mu_0,\mu_1 \in \P_{2,ac}(\R^k)$ and $t \in (0,1)$, let $\mu_t$ denote the pushforward of $\mu_0$ by the map $x \mapsto (1-t)x + t T_{{\mu_0}\to {\mu_1}}(x)$. This curve $(\mu_t)_{t \in [0,1]}$ is the unique \emph{geodesic} that connects $\mu_0$ to $\mu_1$, also known as the \emph{displacement interpolation} \cite[Chapter 5]{villani2021topics}. We will say that a set $S\subset \P_{2,ac}(\R^k)$ is \emph{geodesically convex} if for each $\mu_0,\mu_1 \in S$, the geodesic $(\mu_t)_{t \in [0,1]}$ connecting $\mu_0$ to $\mu_1$ lies entirely in $S$. Similarly, for $\lambda \in \R$, a functional $\Phi:\P_{2,ac}(\R^k) \to\R$ is said to be $\lambda$-geodesically convex if for each $\mu_0,\mu_1 \in \P_{2,ac}(\R^k)$ and for each $t \in [0,1]$ we have
\[\Phi(\mu_t)\leq (1-t)\Phi(\mu_0)+t\Phi(\mu_1)-t(1-t)\frac{\lambda}{2} \W^2_2(\mu_1,\mu_0).\]
For $\lambda=0$ we simply say \emph{geodesically convex}, and if the inequality is always strictly we say \emph{strictly $\lambda$-geodesically convex}.
The definition readily adapts to a functional $\Phi:S \to \R$, where $S \subset \P_{2,ac}(\R^k)$ is a geodesically convex set.

With the above terminology in place, we can explain the precise sense in which the MFVI problem \eqref{MFVI-formal} is a convex optimization problem. We begin by noting that the domain $\P_{2,ac}^{\otimes d}(\R^k)$ is geodesically convex. 
\begin{proposition} \cite[Proposition 3.3]{lacker2023IndependentProjection}
\label{geodesic convexity of product measures}
The set $\P_{2,ac}^{\otimes d}(\R^k) \subset \P_{2,ac}(\R^k)$ is geodesically convex.
\end{proposition}

The proof given in \cite{lacker2023IndependentProjection} treats the case of blocks of equal dimension but clearly generalizes to blocks of different dimension.
The next theorem collects known results on the geodesic convexity of three key functionals on $\P_2(\R^k)$:
\begin{enumerate}
\item The differential entropy $h(\mu)$, defined in \eqref{def:differentialentropy}.
\item The \emph{potential energy} associated with the convex function  $\psi : \R^k \to \R$.
\begin{equation}
\Psi(\mu)= \int_{\R^k} \psi(x)\,\mu(dx). \label{def:Psi}
\end{equation}
\item The relative entropy $H(\mu\,|\,\rho)$, which can be written as $H(\mu\,|\,\rho)= \Psi(\mu)+h(\mu)$ when $\rho(dx)=e^{-\psi(x)}dx$ is a probability measure.
\end{enumerate}

\begin{theorem} \label{th:geoconvex-examples} \cite[Theorem 5.15]{villani2021topics}
\begin{enumerate}
    \item The differential entropy $h$ is geodesically convex.
    \item If $\psi : \R^k \to \R$ is (strictly) $\lambda$-convex with $\lambda\geq0$, then the functional $\Psi$ defined in \eqref{def:Psi} is (strictly) $\lambda$-geodesically convex.    
    \item If $\psi$ is proper and (strictly) $\lambda$-convex with $\lambda\geq0$, and $\rho(x)=e^{-\psi(x)}$ defines a probability density, then the functional $\mu \mapsto H(\mu\,\|\,\rho)$ is (strictly) $\lambda$-geodesically convex.
\end{enumerate}
\end{theorem}

Theorem \ref{th:geoconvex-examples}(3) and Proposition \ref{geodesic convexity of product measures} show the precise sense in which the MFVI problem \eqref{MFVI-P2} is a (strongly) convex optimization problem, when $\psi$ is (strongly) convex: It is the minimization of a geodesically convex functional on a geodesically convex set.

We next explain the product structure of optimal transport maps on $\P_2^{\otimes d}(\R^k)$ and a closely related tensorization identity for the Wasserstein distance between product measures. This is known and straightforward, but we report it for completeness, as it will be used repeatedly throughout the paper.
Recall that for a product measure $\nu$ we write $\nu^1,\ldots,\nu^d$ for its marginals, so that $\nu = \nu^1 \otimes \cdots \otimes \nu^d$.

\begin{lemma}
\label{product structure of Brenier Maps}
For $\nu, \mu \in \P_{2,ac}^{\otimes d}(\R^k)$,  we have
\[T_{\nu \to \mu}(x)= (T_{\nu^1 \to \mu^1}(x^1),\ldots,T_{\nu^d\to\mu^d}(x^d)),\]
as well as
\[
\W_2^2(\mu,\nu) = \sum_{i=1}^d\W_2^2(\mu^i,\nu^i).
\]
\end{lemma}
\begin{proof}
The second claim follows immediately from the first.
By Brenier's theorem, we may write $T_{\nu^i \to \mu^i}(x^i) = \nabla\phi_i (x^i)$, with $\phi_i$ being a convex function. 
    Clearly $x \mapsto (T_{\nu^i \to \mu^i}(x^i))_{i=1}^d$ pushes $\nu$ forward to $\mu$, and it is the gradient of the convex function $x \mapsto \sum_{i=1}^d  \phi_i (x^i)$.
Conclude by the characterization of Brenier map as the unique gradient of a convex function that pushes $\nu$ forward to $\mu$.
\end{proof}

Lemma \ref{product structure of Brenier Maps} shows that the geodesic convexity of $\P_{2,ac}^{\otimes d}(\R^k)$ is reflected on the structure of the Brenier map between two product measures: the $i$th entry of the vector $T_{\nu\to\mu}(x)$ depends only on the $i$th variable $x^i$.
The above argument is taken from \cite[Proposition 3.3]{lacker2023IndependentProjection}, and in fact Proposition \ref{geodesic convexity of product measures} follows quickly from it.

\subsection{Subdifferential calculus}
We next recall some notions of subdifferential calculus in Wasserstein space, borrowed from the classic text \cite{Ambrosio2008-dt}.
\begin{definition}
Let $\Phi:\P_{2}(\R^k)\to \R \cup \{\infty\}$ be proper and lower semi-continuous in the $\mathbb{W}_2$-topology.
    Fix $\nu \in \P_{2,ac}(\R^k)$ such that $|\Phi(\nu)|<\infty$. A vector $\xi \in L^2(\nu)$ is in the \emph{subdifferential} of $\Phi$ at $\nu$ if for every $\mu \in \P_{2,ac}(\R^k)$ we have
    \[\Phi(\mu)-\Phi(\nu)\geq \int_{{\R^k}} \left\langle\xi(x),T_{\nu \to \mu}(x)-x\right\rangle \,\nu(dx) + o(\W_2(\mu,\nu)).\]
    Or equivalently,
    \[\liminf_{\mu \to \nu}\frac{\Phi(\mu)-\Phi(\nu) -\int_{{\R^k}} \left\langle\xi(x),T_{\nu \to \mu}(x)-x\right\rangle \,\nu(dx)}{\W_2(\mu,\nu)}\geq 0,\]
    where the limit is understood in the $\W_2$ sense.
\end{definition}

We will denote by $\wgrad \Phi(\mu,x)$ the $L^2(\nu)$-unique function in the subdifferential of $\Phi$ that solves the \emph{minimal selection problem}
\[\wgrad \Phi(\mu,\cdot) \in \argmin \|\xi\|_{L^2(\nu)},\]
where the argmin is taken over all $\xi$ in the subdifferential of $\Phi$ at $\nu$. For a proof of uniqueness, see \cite[Lemma 10.1.5]{Ambrosio2008-dt}. Notice that this functional is guaranteed to exist as long as $\Phi$ has nonempty subdifferential at $\nu$. We will often call $\wgrad \Phi$ the \emph{Wasserstein gradient} of $\Phi$. 

Proving that a functional has nonempty subdifferential is often difficult. The next three theorems collect  the known results that we will use, which pertain to the same three key functionals discussed in the previous section.

\begin{theorem} \cite[Proposition 10.4.2]{Ambrosio2008-dt}
\label{differentiability of potential energy}
Let $\psi:\R^k \to \R$ be convex, and define $\Psi:\P_2(\R^k)\to\R$ by \eqref{def:Psi}.
Fix $\mu \in \P_{2,ac}(\R^k)$ such that $\psi \in L^1(\mu)$. Then $\Psi$ has non empty subdifferential at $\mu$ if and only if $\nabla \psi \in L^2(\mu)$, and in particular $\wgrad \Psi(\mu,x)=\nabla \psi(x)$, where $\nabla \psi$ is the weak gradient. 
\end{theorem}

The differential entropy $h$ must be handled with more care, and this is a recurring technical obstacle throughout this article. Recall that we abuse notation by using the same letter for both a measure $\mu$ and its density.

\begin{theorem} \cite[Theorem 10.4.6]{Ambrosio2008-dt}
\label{differentiability of internal energy}
If $\mu \in \P_{2,ac}(\R^k)$ has a positive, weakly differentiable density with $\nabla \log \mu \in L^2(\mu)$, then $h$ has nonempty subdifferential at $\mu$, and $\wgrad h(\mu,\cdot)=\nabla \log \mu(\cdot)$.
\end{theorem}

We will also use an analogous result for relative entropy.
\begin{theorem} \cite[Theorem 10.4.9]{Ambrosio2008-dt}
\label{differentiability of relative internal energy}
Let $\nu,\mu \in \P_{2,ac}(\R^k)$. 
If $\mu\slash\nu$ is positive and weakly differentiable with $\nabla \log (\mu\slash\nu)\in L^2 (\mu)$, then the functional $H(\cdot\,\|\,\nu)$ has nonempty subdifferential at $\mu$, and it holds that $\wgrad H(\mu\,\|\,\nu)=\nabla \log (\mu\slash\nu)$.
\end{theorem}
We remark that if $h$ and $\Psi$ have non empty subdifferential at $\eta$, a fortiori $H$ has non empty subdifferential at $\eta$ and Theorem \ref{differentiability of relative internal energy} tells us that $\wgrad H(\eta\,\|\,\nu)(x)= \wgrad \Psi(\eta,x)+\wgrad h(\eta,x)$.

Putting the above machinery to use, we state a well known convexity inequality that will be used extensively in this paper:

\begin{proposition} \cite[Section 10.1.1]{Ambrosio2008-dt} 
\label{characterization of convexity with gradient}
    If $\Phi$ is $\lambda$-geodesically convex and has nonempty subdifferential at $\nu\in \P_{2,ac}(\R^k)$, then for each $\mu\in \P_{2,ac}(\R^k)$ we have
    \[\Phi(\mu)\geq \Phi(\nu)+\int_{\R^k}\left\langle\wgrad \Phi(\nu,x),T_{\nu \to \mu}(x)-x\right\rangle\,\nu(dx)+\frac{\lambda}{2}\W_2^2(\mu,\nu).\]
\end{proposition}

\section{The CAVI algorithm}
\label{CAVI ALGORITHM SECTION}

In this section we introduce notation that will be used throughout the paper and prove some important integrability properties of the CAVI algorithm. In this section, we work under more general assumptions on the measurable function $\psi : \R^k \to \R$ than those of Theorem \ref{main theorem}:

\begin{itemize}
    \item[\textbf{(Q.1)}] We can find $\alpha\in \R$ and $\beta > 0$ such that $\psi(x)\geq\alpha+ \beta|x|$ almost everywhere.
    \item[\textbf{(Q.2)}] There exist $c>0$ and $p \ge 1$ such that $\psi(x) \leq c(1+|x|^p)$ almost everywhere.
    \item[\textbf{(Q.3)}] $\int_{\R^k} e^{-\psi(x)}dx =1$, so that $\rho(x) = e^{-\psi(x)}$ defines a probability density function.
\end{itemize}

\begin{remark} \label{re:convex->Q}
If $\psi$ is convex and satisfies (Q.3) then (Q.1) holds automatically  \cite[Lemma 1]{Cule2010LogConcave}.
Note that as long as $\int_{\R^k} e^{-\psi(x)}dx <\infty$ we may shift $\psi$ to make the normalization constant equal to 1 with no loss in generality.
\end{remark}

For $x=(x^1,\ldots,x^d)$ in $\R^k = \R^{k_1} \times \cdots \times \R^{k_d}$, for $i \in [d]=\{1,\ldots,d\}$, and for $y \in \R^{k_i}$, we use the notation
\begin{align*}
x^{-i} &=(x^1,\ldots,x^{i-1},x^{i+1},\ldots,x^d) \quad \, \in \R^{k-k_i} \\
(y^{i},x^{-i} )&=(x^1,\ldots,x^{i-1},y^i,x^{i+1},\ldots,x^d) \in \R^k.
\end{align*}
Similarly, given $i \in [d]$ and a product measure $\mu=\bigotimes_{i=1}^d \mu^i \in \P^{\otimes d}(\R^k)$ (recalling the notation \eqref{def:productmeasures}), we denote by $\mu^{-i}\in \P(\R^{k-k_i})$ the marginal of $x^{-i}$, i.e.,
\[
\mu^{-i} = \mu^1 \otimes \cdots \otimes \mu^{i-1} \otimes \mu^{i+1} \otimes \cdots \otimes \mu^d.
\]
And given $\nu \in \P(\R^{k_i})$, we write (with a slight abuse of notation)
\[\nu\otimes \mu^{-i} = \mu^1\otimes\cdots\otimes\mu^{i-1}\otimes \nu\otimes \mu^{i+1}\otimes\cdots\otimes \mu^d.\]
If $(\mu_n)_{n \ge 0}$ are the iterates of the CAVI algorithm as defined in \eqref{def:update-explicit}, we define
\[\mu_{n:i} = \mu_{n+1}^1 \otimes \cdots \otimes \mu_{n+1}^i \otimes \mu_{n}^{i+1} \otimes \cdots \otimes \mu_n^d.\]
In words, $\mu_{n:i}$ is the product measure obtained from $\mu_n$ after updating the first $i$ marginals.
We adopt the convention that $\mu_{n:0}=\mu_n$, and note that $\mu_{n:d}=\mu_{n+1}$.
In terms of the marginals,
\[
\mu^j_{n:i} = \begin{cases} \mu_n^j &\text{if } j > i \\ \mu_{n+1}^j &\text{if } j \le i. \end{cases}
\]
With this in mind we can restate \eqref{def:update-explicit} as
\begin{equation}
\mu^i_{n+1}(x^i) \propto \exp\left\{-\int_{\R^{k-k_i}} \psi(x^i,y^{-i})\,\mu_{n:i-1}^{-i}(dy^{-i})\right\}. \label{eq:iterates2}
\end{equation}
From now on, we will use the shorthand $H(\mu)=H(\mu\,\|\,\rho)$.

\subsection{Iterates are well defined}

We begin by showing that the iterates are well defined if $\mu_0$ has finite $p$th moment. Remarkably, even if $\mu_0$ has infinite relative entropy, the iterates from \eqref{def:update-explicit} are well defined and lead to finite entropy $H(\mu_1) <\infty$ after one iteration (through all $d$ coordinates). In particular, this allows for $\mu_0$ to be a Dirac.
Recall that a distribution $\eta$ is said to be \emph{subexponential} if there exists $c>0$ such that $\int e^{c|x|}\,\eta(dx)< \infty$.

\begin{lemma}
\label{form of the iterates lemma}
Assume (Q.1--3).
  If $\mu_0 \in \P_p^{\otimes d}(\R^k)$, then the following hold for each $n \ge 1$:
\begin{enumerate}[(1)]
\item The iterates $\mu_n$ are well defined subexponential distributions.
\item $H(\mu_n)<\infty$.
\item For each $i \in [d]$, $\mu_{n+1}^i$ is the unique minimizer
\begin{align}
    \label{form of the iterates}
        \mu_{n+1}^i = \argmin_{ \eta \in \P(\R^{k_i})} H(\eta\otimes \mu_{n:i-1}^{-i}).
    \end{align}
\item $H(\mu_n) \le H(\mu_{n-1})$, and more generally $H(\mu_{n:i}) \le H(\mu_{n:i-1})$ for $i \in [d]$.
\end{enumerate}    
\end{lemma}
\begin{proof} { \ }
\begin{enumerate}[(1)]
\item We prove that $\mu_1^1$ is a well defined subexponential measure and the rest of claim (1) follows similarly.
We first check that the integration constant is finite and strictly positive. To check that it is strictly positive we use assumption (Q.2):
\begin{align*}
        \exp\left\{-\int_{\R^{k-k_1}} \psi(x^1,y^{-1})\,\mu_0^{-1}(dy^{-1})\right\}  &\geq e^{-c-c |x^1| } 
        \exp\left\{-c\int_{\R^{k-k_1}}|y^{-1}| \,\mu_0^{-1}(dy^{-1})\right\}.
    \end{align*}
     To check that it is finite we use assumption (Q.1):
        \begin{align*}
        \exp &\left\{-\int_{\R^{k-k_1}} \psi(x^1,y^{-1})\,\mu_0^{-1}(dy^{-1})\right\}
        \\
        &\leq  e^{-2^{p-1}\beta|x^1|^p}\exp\left\{ - \alpha -2^{p-1}\beta \int_{\R^{k-k_1}} |y^{-1}|^p\,\mu_0^{-1}(dy^{-1})\right\}.
    \end{align*}
    The right-hand side of each inequality is strictly positive and $dx^1$-integrable because $\mu_0$ was assumed to have finite $p$th moment.
    Thus, $\mu_1^1$ is well defined.
    The second inequality also shows that $\mu_1^1$ is subexponential, and in particular the integrability of order $p$ is preserved along the iterates. 
    \item Letting $Z_1$ be the normalization constant of $\mu_1$, we next check that 
    \begin{align*}
        H(\mu_1)&=\int_{\R^k} \psi(x)\,\mu_1(dx)-\log Z_1 - \sum_{i=1}^d \int_{\R^{k_i}}\int_{\R^{k-k_i}} \psi(x^i,y^{-i})\,\mu^{-i}_{0:i-1}(dy^{-i})\,\mu_1^i(dx^i)\\
        &\leq c\int_{\R^k}(1+|x|^p)\,\mu_1(dx) - \log Z_1 -d\alpha  < \infty.
    \end{align*}
\item We check the case $n=i=1$ by showing that $\mu_2^1=\mathrm{argmin}_\eta H(\eta \otimes \mu_1^{-1})$, and the rest will follow similarly. Let $f(x^1)= \int_{\R^{k-k_1}} \psi(x^1,y^{-1})\,\mu_2^{-1}(dy^{-1})$, so that $\mu_2^1(x^1)= (1/Z_2) e^{-f(x^1)}$ by definition for some constant $Z_2$.
For $\eta \in \P(\R^{k_1})$, using the identity \eqref{eq:Hcomposite}, 
\begin{align*}
H(\eta \otimes \mu_1^{-1}) &= h(\eta) + \sum_{i=2}^d h(\mu_1^i) + \int_{\R^{k_1}}\int_{\R^{k-k_1}} \psi(x)\,\mu_1^{-1}(dx^{-1})\eta(dx^1) \\
	&= h(\eta) + \sum_{i=2}^d h(\mu_1^i) + \int_{\R^{k_1}}f(x^1)\,\eta(dx^1) \\
	&= H(\eta\,\|\,\mu_2^1)  + \sum_{i=2}^d h(\mu_1^i) - \log Z_2.
\end{align*}    
Because  $H(\eta\,\|\,\mu_2^1)$ is uniquely minimized by $\eta=\mu_2^1$, we deduce that $H(\eta \otimes \mu_1^{-1})$ is uniquely minimized by the same choice.
\item This follows immediately from (3).
\end{enumerate}   \vskip-0.5cm
\end{proof}

\subsection{Uniform moment bounds}
Lemma \ref{form of the iterates lemma} tells us that $\mu_n$ is subexponential for every $n$, but not necessarily with the same parameter.  We obtain uniform estimates in the next Lemma.

\begin{lemma}
\label{moment bounds}
Assume (Q.1--3).
    For each $q \geq 1$ and  $n \in \N$,
    \[\int_{\R^{k}} |x|^q\,\mu_n(dx)\leq e^{(2d+1)H(\mu_1)-(d+1)\alpha}\left(\int_{\R^{k}} e^{-(\beta|x|+\alpha)\slash2}\,dx\right)^{2d} \int_{\R^k} |x|^q e^{-\beta\sum_i |x^i|}\,dx<\infty.\]
\end{lemma}
\begin{proof}
Recalling $H(\cdot) = H(\cdot\,\|\,\rho)$, define
    \[\mathcal{S}= \left\{\mu \in \P_2^{\otimes d}(\R^k) : H(\mu)\leq H(\mu_1)\right\}.\]
 Lemma \ref{form of the iterates lemma}(4) ensures that entropy is nonincreasing along iterates, which implies $\mu_n \in \mathcal{S}$ and $\mu_{n:i} \in \mathcal{S}$ for each $n \ge 1$ and $i \in [d]$.
    Using the variational characterization of entropy (see, for example, \cite[Proposition 2.3(b)]{BudhirajaDupuis}) along with  (Q.1), we find for each $\mu \in \mathcal{S}$ that
    \begin{align}
        \int_{\R^{k}} \psi(x)\,\mu(dx) &\leq 2\h(\mu \,\|\,\rho)+2\log \int_{\R^{k}}e^{ \psi(x)/2}\,\rho(dx) \le I, \label{pf:Ibound} \\
\text{ where }  \ \        I &=  2H(\mu_1) + 2\log\int_{\R^{k}} e^{-(\beta|x|+\alpha)/2}\,dx < \infty. \nonumber
    \end{align}
We next prove that the integration constant $Z_n$ of $\mu_n$ is uniformly bounded away from zero.
     By using the form of $\mu_n$ we obtain
    \[H(\mu_n)=\int_{\R^{k}} \psi(x)\,\mu_n(dx)-\log Z_n - \sum_{i=1}^d\int_{\R^{k_i}}\int_{\R^{k-k_i}}\psi(x^i,y^{-i})\,\mu^{-i}_{n-1:i-1}(dy^{-i})\,\mu^i_n(dx^i)\]
 Observe that $\mu_n^i \otimes \mu_{n-1:i-1}^{-i} = \mu_{n-1:i}$ for each $i$. Apply \eqref{pf:Ibound} to  $\mu_{n-1:i} \in \mathcal{S}$ to get
\begin{align*}
    \int_{\R^{k_i}} &\int_{\R^{k-k_i}}\psi(x^i,y^{-i})\,\mu^{-i}_{n-1:i-1}(dy^{-i})\,\mu^i_n(dx^i) = \int_{\R^{k}} \psi(x) \,\mu_{n-1:i}(dx) \leq I.
\end{align*}
Notice that this bound is uniform in $n$ and $i$. 
 Using also $\psi\ge \alpha$, we obtain
\begin{align*}
    H(\mu_1)\geq H(\mu_n)\geq \alpha -\log Z_n - dI .
\end{align*}
We deduce
\[
Z_n \geq \exp \big(\alpha - dI-H(\mu_1) \big)>0.
\]
Finally, let
\[f(x)=\sum_{i=1}^d\int_{\R^{k-k_i}} \psi(x^i,y^{-i})\,\mu_{n:i-1}(dy^{-i})\]
and notice that (Q.1) implies $f(x)\geq \beta \sum_{i=1}^d|x^i|+d\alpha$. Then we have
\begin{align*}
\int_{\R^k} |x|^q\,\mu_n(dx) = \frac{1}{Z_n}\int_{\R^k}|x|^q e^{-f(x)}\,dx \leq e^{dI + H(\mu_1) - (d+1)\alpha }\int_{\R^k} |x|^q e^{-\beta\sum_i |x^i|}\,dx.
\end{align*}
The right-hand side is bounded in $n$, and the proof is complete upon substituting $I$.
\end{proof}

A useful corollary is that the limit points of the iterates enjoy some continuity properties:

\begin{corollary}
\label{continuity of integral on sublevel sets}
The sequence $(\mu_n)$ is tight.
If a subsequence $(\mu_{n_t})$ converges weakly to some $\mu_*$, then 
    \[\int_{\R^k}\psi(x)\,\mu_{n_t}(dx) \to \int_{\R^k} \psi(x)\,\mu(dx).\]
\end{corollary}
\begin{proof}
Tightness follows from the fact that $\sup_{n \ge 1} H(\mu_n\,\|\,\rho) < \infty$ by Lemma \ref{form of the iterates lemma}(2,4), and from the compactness of level sets of relative entropy \cite[Lemma 2.4(c)]{BudhirajaDupuis}.
Let $\mu_{n_t} \to \mu_*$. Because $(\mu_{n_t})$ is contained in a level set of relative entropy, it follows that $\int \varphi \,d\mu_{n_t} \to \int\varphi\,d\mu_*$ for any bounded measurable $\varphi$; see \cite[Lemma 2.5(a)]{BudhirajaDupuis}.
From (Q.1,2) and Lemma \ref{moment bounds} we have $\sup_{n \in \N}\int_{\R^k}\psi^2 \,d\mu_n  < \infty$, and the claim follows by a standard uniform integrability argument.
\end{proof}

\section{Qualitative Convergence}\label{qualitative section}

In this section we maintain assumptions (Q.1--3) from the previous section, as well as the abbreviation $H(\cdot)=H(\cdot\,\|\,\rho)$.
For very general $\psi$ the CAVI algorithm need not converge at all, much less to a minimizer. We can nonetheless characterize the limit points of the sequence of iterates as \emph{stationary points}. A measure $\mu \in \P^{\otimes d}_p(\R^k)$ is said to be a stationary point if $H(\mu) < \infty$ and if for each block $i\in [d]$ and each $\nu \in \P_p(\R^{k_i})$ it holds  that
\[H(\nu\otimes\mu^{-i}) \geq H(\mu).\] 
The restriction to measures with finite $p$ moments is convenient as it ensures $\psi$ is integrable under the assumptions (Q.1--3), and it is not a restriction in light of Lemma \ref{moment bounds}.

Intuitively, if the iterates $\mu_n$ (or a subsequence) converge, they should converge to a stationary point. In this section, we prove in Proposition \ref{qualitative converegence} that this is indeed the case under assumptions (Q.1--3) by exploiting the tractable structure of the updates.
The natural follow up question is if stationary points are minimizers, and in  in Proposition \ref{differentiability at stationary points} below we give an affirmative answer when $\psi$ is convex, by exploiting a well known characterization of stationary points in terms of the \emph{mean field equations} (Proposition \ref{mean field equation}).

\begin{proposition}
\label{qualitative converegence}
Assume (Q.1--3). The limit points of the tight sequence $(\mu_n)$ are stationary points of $H$.
\end{proposition}

The proof will adapt techniques developed in \cite{QualitativeConvergenceTseng2001} for the Euclidean case, exploiting heavily the composite structure of $H(\cdot)$ noted in \eqref{eq:Hcomposite}. Corollary \ref{continuity of integral on sublevel sets} lets us take limits of $\int_{\R^k}\psi \,d\mu_n$, but the differential entropy $h(\mu_n)$ is not continuous. Nonetheless, we will be able to exploit its separable dependence on the marginals as well as its lower semicontinuity and strict convexity (in the standard sense, not the geodesic sense).

\begin{proof}[Proof of Proposition \ref{qualitative converegence}]
Tightness was noted in Corollary \ref{continuity of integral on sublevel sets}.
Let $\mu_{n_t}$ be a subsequence converging weakly to some limit $\lambda$. We wish to prove that $\mu_{{n_t}+1}$, up to extracting a subsequence, also converges to the same limit $\lambda$. To begin, we will first prove that, up to extracting a subsequence, $\mu_{n_t:1}$ converges to the same limit. 
Suppose for the sake of contradiction that $\mu_{n_t:1}\to \sigma \neq \lambda$. Let 
\[\eta_{t} = \frac{1}{2} (\mu_{{n_t}}+\mu_{{n_t:1}})=\frac{1}{2}\mu_{{n_t}}^1 \otimes \mu_{n_{t}}^{-1}+\frac{1}{2}\mu_{n_{t}+1}^1\otimes \mu_{n_t}^{-1}=\frac{1}{2}\left(\mu_{n_t}^1+\mu_{n_{t}+1}^1\right)\otimes \mu_{n_t}^{-1}.\]
Note that $\eta_{t}$ is in $\P_p^{\otimes d}(\R^k)$, and that $\eta=\frac{1}{2}(\lambda+\sigma)$ is the limit of $\eta_t$. Note also that $\mu_{n_t}$ and $\mu_{n_t:1}$ have the same $i$th marginal for each $i \ge 2$, and so $\sigma^i=\lambda^i$ for all $i \ge 2$ as well.

Differential entropy is known to be lower semicontinuous  along sequences with uniformly bounded second moments; this is folklore, but see \cite[Lemma A2]{liu2018forward} for a concise proof (noting they use the information theory sign convention for entropy). Combined with Corollary \ref{continuity of integral on sublevel sets}, we deduce that
\[
\int_{\R^k}\psi\,d\lambda  + h(\lambda^1) \leq \liminf_{t \to \infty}\int_{\R^k}\psi \,d\mu_{n_t} + h(\mu_{n_t}^1).
\]
Suppose (again up to a subsequence) that 
\begin{equation}
\int_{\R^k}\psi \,d\mu_{n_t}  + h(\mu_{n_t}^1) \to \ell. \label{pf:qual1}
\end{equation}
Using the identity \eqref{eq:Hcomposite} and the fact that $\mu_{n_t}=\mu_{n_t:1}$ have the same $i$th marginal for $i \ge 2$, observe that 
\[\int_{\R^k}\psi \,d\mu_{n_t} + h(\mu_{n_t}^1)-\int_{\R^k}\psi \,d\mu_{n_t:1} + h(\mu_{n_{t}+1}^1)=H(\mu_{n_t})-H(\mu_{n_t:1})\to 0,\]
the limit following from the fact that $H(\cdot)$ decreases with every update by Lemma \ref{form of the iterates lemma}(4).
This implies 
\begin{equation}
\int_{\R^k}\psi \,d\mu_{n_t:1} + h(\mu_{n_{t}+1}^1) \to \ell. \label{pf:qual2}
\end{equation}
By convexity and again Lemma \ref{form of the iterates lemma}(4), we have $H(\eta_t) \leq H(\mu_{n_t})$ for every t, or equivalently
\[
\int_{\R^k}\psi \,d\eta_t + h(\eta_t^1) \leq \int_{\R^k}\psi \,d\mu_{n_t}  + h(\mu_{n_t}^1).
\] 
Hence,
\begin{equation}
\int_{\R^k}\psi \,d\eta + h(\eta^1) \leq \liminf_{t\to \infty}\int_{\R^k}\psi \,d\eta_t + h(\eta_t^1)\leq \ell. \label{pf:qual3}
\end{equation}
We will prove that the inequality \eqref{pf:qual3} is actually an equality. Suppose for the sake of contradiction that $\int_{\R^k}\psi \,d\eta + h(\eta^1)< \ell$. Then, by Corollary \ref{continuity of integral on sublevel sets} and \eqref{pf:qual2},
\begin{align*}
\lim_{t \to \infty}\int_{\R^k}\psi \,d(\eta^1 \otimes \mu_{n_t}^{-1})  + h(\eta^1) &= \int_{\R^k}\psi \,d\eta + h(\eta^1) \\
	&< \ell = \lim_{t \to \infty} \int_{\R^k}\psi \,d\mu_{n_t:1} + h(\mu_{n_{t}+1}^1).
\end{align*}
Thus, we can find $T$ large enough such that for every $t \geq T$ we obtain
\[\int_{\R^k}\psi \,d(\eta^1 \otimes \mu_{n_t}^{-1})  + h(\eta^1) < \int_{\R^k}\psi \,d\mu_{n_t:1} +h(\mu_{n_t:1}^1). \]
Adding $\sum_{i \ge 2} h(\mu_{n_t}^i)$ to both sides we find, for $t \ge T$,
\[H(\eta^1 \otimes \mu_{t}^{-1}) < H(\mu_{n_t:1}). \]
This contradicts the optimality of $\mu_{n_t:1}$ from Lemma \ref{form of the iterates lemma}(3), and we deduce that \eqref{pf:qual3} collapses to equality.
Using strict convexity of $h$ (in the usual sense) followed by Corollary \ref{continuity of integral on sublevel sets} and lower semicontinuity of entropy, we find
\begin{align*}
    \ell = \int_{\R^k}\psi \,d\eta + h(\eta^1) &< \frac{1}{2}\left(\int_{\R^k}\psi \,d\sigma  +h(\sigma^1)\right)+ \frac{1}{2}\left(\int_{\R^k}\psi \,d\lambda +h(\lambda^1)\right)\\ 
    &\leq \lim_{t \to \infty} \frac{1}{2}\left(\int_{\R^k}\psi \,d\mu_{n_t:1}  +h(\mu^1_{n_t:1})\right)+ \frac{1}{2}\left(\int_{\R^k}\psi \,d\mu_{n_t} +h(\mu_{n_t}^1)\right).
\end{align*}
The right-hand side equals $\ell$ according to  \eqref{pf:qual1} and \eqref{pf:qual2}, which yields
a contradiction. It must be then that $\sigma=\lambda$.

Now that we know that $\mu_{n_t}$ and $\mu_{n_t:1}$ have the same limit $\lambda$ (up to a subsequence), we can repeat the same argument to find that $\mu_{n_t:2}$ has the same limit, and so on, ultimately showing that $\mu_{n_t+1} \to \lambda$. We will finally show that $\lambda$ is a stationary point. Fix a block $i$ and an arbitrary $\eta \in \P_p(\R^{k_i})$. By Lemma \ref{form of the iterates lemma}(4) we have $H(\mu_{n_t:i}) \le H(\eta \otimes\mu_{n_t:i}^{-i})$, and subtracting $\sum_{j \neq i}h(\mu^j_{n_t:i})$ from both sides (again recalling \eqref{eq:Hcomposite}) implies
\[\int_{\R^k}\psi \,d\mu_{n_t:i} +h(\mu_{n_t+1}^i) \leq \int_{\R^k}\psi \,d(\eta \otimes \mu_{n_t:1}^{-i}) + h(\eta^i)\]
By Corollary \ref{continuity of integral on sublevel sets} and lower semicontinuity of entropy, we get
\begin{align*}
    \int_{\R^k}\psi \,d\lambda + h(\lambda^i) \leq \liminf_{t\to \infty}\int_{\R^k}\psi \,d\mu_{n_t:1} +h(\mu_{n_t+1}^i)\leq  \int_{\R^k}\psi \,d(\eta \otimes \lambda^{-i}) +h(\eta ).
\end{align*}
Adding $\sum_{j \neq i} h(\lambda^j)$ to both sides we obtain  $H(\lambda) \leq H(\eta \otimes \lambda^{-i})$, which proves that $\lambda$ is a stationary point.
\end{proof}

We remark that Proposition \ref{qualitative converegence} requires only mild integrability conditions on $\psi$ and $\mu_0$. The conclusion is accordingly weak. Nonetheless, the following proposition provides some structure of stationary points, by characterizing stationary points as solutions of the so-called \emph{mean field equation}. This characterization is well known but perhaps not documented at this level of generality, so we present some details; it boils down to the Gibbs variational principle after integrating out the other blocks.

\begin{proposition}
\label{mean field equation}
Assume (Q.1--3), and let $\mu \in \P_p^{\otimes d}(\R^k)$. Then $\mu$ is a stationary point if and only if it satisfies the following \emph{mean field equation}:
\begin{equation}
\mu^i(x^i) \propto \exp\left\{-\int_{\R^{k-k_i}} \psi(x^i,y^{-i})\,\mu^{-i}(dy^{-i})\right\}, \qquad \forall i\in [d] . \label{eq:MFeq}
\end{equation}
In this case, the measure $\mu$ is subexponential.
\end{proposition}
\begin{proof}
    Assume that $\mu$ is stationary and fix a coordinate $i \in [d]$. Let 
    \begin{equation}
    f_i(x^i)= \int_{\R^{k-k_i}} \psi(x^i,x^{-i})\,\mu^{-i}(dx^{-i}). \label{def:fi}
    \end{equation}
    By definition, for each $\nu^i \in \P_p(\R)$ we have $H(\mu) \leq H(\nu \otimes \mu^{-i})$. Because $H(\mu) < \infty$ by definition of a stationary point and because $\psi \in L^p(\nu \otimes \mu^{-i})$ by (Q.2), it follows that $H(\nu \otimes \mu^{-i}) < \infty$ if and only if $h(\nu) < \infty$. Hence, using the composite structure \eqref{eq:Hcomposite}, the inequality $H(\mu) \leq H(\nu \otimes \mu^{-i})$ is then equivalent to
    \begin{equation}
    \int_{\R^{k_i}} f_i \,d\mu^i + h(\mu^i) \leq \int_{\R^{k_i}} f_i \,d\nu +h(\nu), \label{pf:MFeq1}
    \end{equation}
    for every $\nu \in \P_p(\R^{k_i})$ such that $h(\nu) < \infty$.
Define $\eta^i \in \P(\R^{k_i})$ by $\eta^i(dx^i) \propto e^{-f_i(x^i)}dx^i$, which is well defined by (Q.1,2), and note that $\eta^i$ has finite $p$th moment by the same argument as in Lemma \ref{form of the iterates lemma}(1). The inequality \eqref{pf:MFeq1} then rewrites as $H(\mu^i\,\|\,\eta^i) \le H(\nu\,\|\,\eta^i)$ for all $\nu \in \P_p(\R^{k_i})$. This implies that $\mu^i=\eta^i$, which is exactly \eqref{eq:MFeq}.

Assume instead that $\mu$ satisfies \eqref{eq:MFeq}. Then, if initialized at $\mu_0=\mu$, the CAVI iteration is constant, $\mu_n=\mu$ for all $n$. Stationarity then follows immediately from Lemma \ref{form of the iterates lemma}(4), and the subexponential claim follows from Lemma \ref{form of the iterates lemma}(1). 
\end{proof}

The following simple corollary is worth recording:

\begin{corollary}
Assume (Q.1--3).
If the mean field equation  \eqref{eq:MFeq} has a unique solution $\mu_*$, then $\mu_*$ is the unique solution to the MFVI problem  \eqref{MFVI-formal}, and $\mu_n \to \mu_*$.
\end{corollary}

Another useful consequence of Proposition \ref{mean field equation} is that it implies some regularity for stationary points. In turn, this regularity allows us to conclude that stationary points are minimizers when $\psi$ is convex. In what follows, recall that convex functions always admit locally bounded weak derivatives.
Denote by $\nabla_i \psi$ the weak gradient with respect to the variables in block $i$. 

\begin{proposition}
\label{differentiability at stationary points}
Under the assumptions of Theorem \ref{main theorem},  $\mu \in \P_p^{\otimes d}(\R^k)$ is a stationary point if and only if it is a minimizer for the MFVI problem \eqref{MFVI-formal}. 
\end{proposition}
\begin{proof}
Every minimizer is clearly a stationary point, so we focus on the nontrivial implication.
Note that the assumptions of Theorem \ref{main theorem} imply (Q.1--3), by Remark \ref{re:convex->Q}, up to shifting $\psi$ by a scalar for (Q.3).
    Let $\mu \in \P_p^{\otimes d}(\R^k)$ be a stationary point.
    We first check that $h$ has nonempty Wasserstein subdifferential at $\mu$, by checking the assumption of Theorem \ref{differentiability of internal energy}. Let 
    \[f(x)= \sum_{i=1}^d\int_{\R^{k-k_i}} \psi(x^i,y^{-i})\,\mu^{-i}(dy^{-i}), \quad x  \in \R^k. \]
    Note that $f$ is bounded from below by (Q.1) and is finite everywhere because $\mu$ has finite $p$th moment. 
    By Proposition \ref{mean field equation}, $\mu$ satisfies the the mean field equation, which writes as $\mu \propto e^{-f}$.  In particular, $\mu$ is strictly positive.

    Convexity of $\psi$ easily implies that $f$ is convex and thus weakly differentiable.
    We claim that the weak gradient is given by
    \begin{equation}
    \nabla_if(x) = \int_{\R^{k-k_i}} \nabla_i \psi(x^i,y^{-i}) \,\mu^{-i}(dy^{-i}). \label{eq:fderivative}
    \end{equation}
    This is formally clear by exchanging integration and derivative, and to check it rigorously we perform the following calculation for a smooth vector field $\varphi=(\varphi^1,\ldots,\varphi^d) : \R^k \to \R^k$ of compact support:
    \begin{align*}
        \int_{\R^k} f(x) \text{ div }\varphi(x)\,dx &= \sum_{i=1}^d\int_{\R^{k}}\left(\int_{\R^{k-k_i}} \psi(x^i,y^{-i})\text{ div }\varphi(x)\,\mu^{-i}(dy^{-i})\right)\,dx \\
        &=\sum_{i=1}^d\int_{\R^{k-k_i}}\left(\int_{\R^{k} }\psi(x^i,y^{-i})\text{ div }\varphi(x)\,dx\right)\,\mu^{-i}(dy^{-i})\\
        &=-\sum_{i=1}^d\int_{\R^{k-k_i}}\left(\int_{\R^{k}} \left\langle \nabla_i\psi(x^i,y^{-i}) , \varphi^i(x)\right\rangle\,dx\right)\,\mu^{-i}(dy^{-i}) \\
        &= - \int_{\R^k}\sum_{i=1}^d\left\langle\int_{\R^{k-k_i}} \nabla_i \psi(x^i,y^{-i}) \,\mu^{-i}(dy^{-i}), \varphi^i(x)\right\rangle\,dx.
    \end{align*} 
    We use Fubini for the second and fourth equality, with the third owing to the definition of weak gradient of $\psi$. To justify that Fubini's theorem applies to both cases, we note that for every compact set $S\subset\R^{k}$ assumption \eqref{asmp:growth-psi} implies
    \[ \int_{\R^{k-k_i}}\int_S | \nabla_i\psi(x^i,y^{-i})|\,dx\,\mu^{-i}(dy^{-i}) \le \int_{\R^{k-k_i}}\int_S c(1+ |(x^i,y^{-i})|^p)\,dx\,\mu^{-i}(dy^{-i}), \]
    which is finite because $\mu$ has finite $p$th moment. The same holds with $|\psi|$ in place of $|\nabla\psi|$.
    This proves \eqref{eq:fderivative}.
The chain rule for weak derivatives  entails that $\mu=e^{-f}$ is weakly differentiable with $\nabla \mu =  -\nabla f \cdot e^{-f}$; this chain rule is normally stated for Lipschitz functions of $f$, but it applies nonetheless because $f$ is bounded from below.
Finally,  using \eqref{eq:fderivative} and Jensen's inequality,
\begin{align*}
\int_{\R^k} |\nabla \log \mu|^2\,d\mu &= \int_{\R^k} |\nabla f |^2\,d\mu = \int_{\R^k} \sum_{i=1}^d\bigg|\int_{\R^{k-k_i}} \nabla_i \psi(x^i,y^{-i}) \,\mu^{-i}(dy^{-i})\bigg|^2\,\mu(dx)\\
	&\le \int_{\R^k}|\nabla\psi(x)|^2 \,\mu(dx) \le \int_{\R^k}c^2(1+|x|^p)^2 \,\mu(dx).
\end{align*}
This is finite because $\mu$ is subexponential by Proposition \ref{mean field equation}.
Because also $\nabla\psi\in L^2(\mu)$,
we finally deduce  from Theorem \ref{differentiability of relative internal energy} that $H(\nu)=h(\nu)+\int\psi\,d\nu$ has nonempty subdifferential at $\nu=\mu$, with $\wgrad H(\mu,\cdot)=\nabla\psi +\nabla\log\mu$.

To complete the proof, let $\nu \in \P_2^{\otimes d}(\R^k)$, and note that $H$ is geodesically convex by Theorem \ref{th:geoconvex-examples}. We may thus apply Lemma \ref{characterization of convexity with gradient} and then Lemma \ref{product structure of Brenier Maps} to get
     \begin{align*}
         H(\nu)&\geq H(\mu)+\int_{\R^{k}} \left\langle \wgradp H(\mu,x),T_{\mu \to \nu}(x)-x\right\rangle\,\mu(dx)\\
         &= H(\mu)+\sum_{i=1}^d \int_{\R^k}\left\langle \nabla_i \psi(x)+\nabla_i\log \mu^i (x^i),T^i_{\mu^i \to \nu^i}(x^i)-x^i\right\rangle\,\mu(dx)\\
         &=H(\mu)+\sum_{i=1}^d \int_{\R^{k_i}}\!\!\left\langle\int_{\R^{k-k_i}} \nabla_i \psi(x^i,y^{-i})\,\mu^{-i}(dy^{-i})+\nabla_i\log \mu^i (x^i),T^i_{\mu^i \to \nu^i}(x^i)-x^i\right\rangle \mu^i(dx^i)\\
         &=H(\mu),
     \end{align*} 
     with the last step following from the mean field equation \eqref{eq:MFeq}.
\end{proof}

\begin{proof}[Proof of Theorem \ref{main theorem}(1,2)]
Theorem \ref{main theorem}(1) is an immediate consequence of Propositions \ref{differentiability at stationary points} and \ref{qualitative converegence}. Theorem \ref{main theorem}(2) also follows from the above considerations: Strict convexity of $\psi$ ensures that $H$ is geodesically strictly convex, by Theorem \ref{th:geoconvex-examples}, and it thus has at most one minimizer on the geodesically convex set $\P_2^{\otimes d}(\R^k)$ (see Proposition \ref{geodesic convexity of product measures}).
\end{proof}

\section{The Lipschitz gradient case}\label{Lipschitz Gradient}

This section is devoted to the proof of Theorem \ref{main theorem}(3), and we impose the assumptions thereof throughout this section.
The proof relies on two key ingredients. The first is that the objective function $H$ has non empty subdifferential along the iterates $\mu_n$, along with a formula for its Wasserstein gradient. The second point is a simple lifting of a classical convexity inequality from Euclidean to Wasserstein space, which will yield an estimate of \emph{sufficient descent}. With these pieces in place, we find upper and lower bounds for $H(\mu_n)-H(\mu_*)$, where $\mu_*$ is a minimizer, which we can iterate to obtain the claimed rate of convergence.

\begin{lemma}
\label{differentiability of entropy}
    For every $n \ge 1$ and every $i \in [d]$, $h$ has nonempty subdifferential at $\mu_n^i$, $H$ is has nonempty subdifferential at $\mu_{n:i}$, and
    \begin{align*}
    \wgrad h(\mu_{n+1}^i,x^i) &=  \nabla_i \log \mu_{n+1}^i(x^i)= -\int_{\R^{k-k_i}} \nabla_i \psi(x^i,y^{-i})\,\mu_{n:i}^{-i}(dy^{-i}),  \\
    \wgrad H(\mu_{n:i},x) & = \bigg(\nabla_j\psi(x) + \nabla_j \log \mu_{n:i}^{j}(x^j)\bigg)_{j=1,\ldots,d}.
    \end{align*} 
\end{lemma}

To be completely clear, the first identity in Lemma \ref{differentiability of entropy} is for the Wasserstein gradient of $h$ viewed as a functional on $\P_2(\R^{k_i})$, which is a vector field $\R^{k_i} \ni x^i \mapsto \wgrad h(\mu_{n+1}^i,x^i)  \in \R^{k_i}$. The second identity is for the Wasserstein gradient of the functional $H$ on $\P_2(\R^k)$, which is a vector field $\R^k \ni x \mapsto \wgrad H(\mu_{n:i},x)  \in \R^{k}$. Note also that when $i=d$ the second identity becomes 
\begin{align}
\wgrad H(\mu_{n+1},x)  &= \bigg(\nabla_i\psi(x) + \nabla_i \log \mu_{n+1}^i(x^i)\bigg)_{i=1,\ldots,d}  \nonumber \\
	&= \bigg(\nabla_i\psi(x) - \int_{\R^{k-k_i}} \nabla_i \psi(x^i,y^{-i})\, \mu_{n:i}^{-i}(dy^{-i})\bigg)_{i=1,\ldots,d}. \label{eq:wgradH}
\end{align}

\begin{proof}[Proof of Lemma \ref{differentiability of entropy}]
Formally differentiating the formula \eqref{eq:iterates2} for the iterates $\mu_{n+1}^i$ yields the identity
    \begin{equation*}
    \nabla_i \log \mu_{n+1}^i(x^i) = -\int_{\R^{k-k_i}} \nabla_i\psi(x^i,y^{-i})\,\mu_{n:i}^{-i}(dy^{-i}).
    \end{equation*}
(Note that $\mu_{n:i}^{-i}=\mu_{n:i-1}^{-i}$.)
    To justify this rigorously, in the sense of weak gradient, one can argue exactly as in the proof of Theorem \ref{differentiability at stationary points}.
Since $\nabla\psi$ is $L$-Lipschitz, it must be that $|\nabla\psi(x)|\le C(1+|x|)$ and  $\psi(x)\leq C(1+|x|^2)$ for some $C$, and we deduce that $\psi \in L^2(\mu_{n+1})$ and thus $\nabla_i\log \mu_{n+1}^i \in L^2(\mu_{n+1}^i)$.
We can then apply Theorem \ref{differentiability of internal energy} to conclude that $h$ has nonempty subdifferential at $\mu_{n+1}^i$, with $\wgrad h(\mu_{n+1}^i,x^i)=\nabla_i \log \mu_{n+1}^i(x^i)$. For the same reasons, we may also apply Theorem \ref{differentiability of relative internal energy} to deduce that $H$ has non empty subdifferential at $\mu_{n:i}$, with
\begin{align*}
\wgrad H(\mu_{n:i},x) &= \nabla\psi(x) + \nabla \log \mu_{n:i}(x) = \nabla\psi(x) + (\nabla_j\log \mu_{n:i}^{j}(x^j))_{j=1,\ldots,d} . \qedhere
\end{align*}
\end{proof}

In the following, for a measurable map $g=(g^1,\ldots,g^d)$ from $\R^k = \R^{k_1}\times \cdots \times \R^{k_d}$ to itself, and for $\nu \in \P(\R^k)$, we use the standard norm for the vector-valued $L^2(\nu)$ space:
\begin{equation*}
\|g\|_{L^2(\nu)}^2 = \sum_{i=1}^d \int_{\R^k} |g^i(x)|^2\,\nu(dx).
\end{equation*}
To exploit the Lipschitz gradient assumed of $\psi$, we will use the following ``lift" of a classical Euclidean convexity inequality.

\begin{lemma} 
\label{Nesterov inequality}
For any $\mu,\nu \in \P_{2,ac}(\R^k)$,
    \[\int_{\R^k} \psi \,d(\mu-\nu) \geq \int_{\R^k}\left\langle\nabla\psi(x),T_{\nu \to \mu}(x)-x\right\rangle\,\nu(dx)+\frac{1}{2L}\|\nabla \psi\circ T_{\nu \to \mu}-\nabla \psi\|^2_{L^2(\nu)}.\]
\end{lemma}
\begin{proof}
As $\psi$ is convex with $L$-Lipschitz gradient, the following inequality is well known \cite[Theorem 5.8]{beck2017first}:
\[\psi(y) - \psi(x)\geq  \langle \nabla \psi(x),y-x\rangle + \frac{1}{2L}|\nabla \psi(y)-\nabla \psi(x)|^2.\]
Take $y=T_{\nu\to\mu}(x)$ and integrate against $\nu(dx)$ to prove the claim.
\end{proof}

The above inequality actually has an analogue for more general geodesically convex functionals in Wasserstein space, but this simple version suffices for our purposes

The main line of the proof of Theorem \ref{main theorem}(3) is inspired by the study of the Euclidean case in \cite{QuantitativeConvergenceHong2016}. The high level idea is to find a way to bound $H(\mu_n)-H(\mu_*)$ from below to estimate how much the function value decreases at each iteration; this is done with the help of Lemma \ref{Nesterov inequality}. We will likewise upper bound $H(\mu_{n})-H(\mu_*)$ to estimate how many more iterations are needed after $n$, and then match the bounds to obtain a rate of convergence. 

\begin{proposition} 
    \label{lower bound lipschitz case}
    For $n \ge 1$,
    \[H(\mu_n)-H(\mu_{n+1}) \geq \frac{1}{2L}\sum_{i=1}^d \|\nabla \psi\circ T_i - \nabla \psi\|^2_{L^2(\mu_{n:i})},\]
    where $T_i := T_{\mu_{n:i} \to \mu_{n:i-1}}$ is the Brenier map from $\mu_{n:i}$ to $\mu_{n:i-1}$.
\end{proposition}
\begin{proof} 
Write $T_i(x)=(T_i^1(x^1),\ldots,T_i^d(x^d))$ for the coordinates, where $T^j_i$ is the Brenier map from $\mu_{n:i}^j$ to $\mu_{n:i-1}^j$, with the separable structure coming from Lemma \ref{product structure of Brenier Maps}. A key point is that $\mu_{n:i}$ and $\mu_{n:i-1}$ have the same marginals except in coordinate $i$, which means that $T^j_i(x^j)=x^j$ for all $j \neq i$.

Let $i \in [d]$, and note that $T_i^i$ is the Brenier map from $\mu_{n+1}^i=\mu_{n:i}^i$ to $\mu_n^i=\mu_{n:i-1}^i$. Combine Proposition \ref{characterization of convexity with gradient} and Lemma \ref{differentiability of entropy} to get
\begin{align*}
h(\mu_{n}^i) - h(\mu_{n+1}^i) &\ge \int_{\R^{k_i}} \langle \wgrad h(\mu_{n+1}^i,x^{i}), T_i^{i}(x^{i})-x^{i}\rangle\,\mu_{n+1}^i(dx^{i}) \\
	&= -\int_{\R^{k_i}}  \bigg(\int_{\R^{k-k_i}} \langle\nabla_i \psi(x^i,y^{-i}), T_i^{i}(x^{i})-x^{i}\rangle\,\mu_{n:i}^{-i}(dy^{-i})\bigg) \,\mu_{n+1}^i(dx^{i}) \\
	&= - \int_{\R^k} \langle \nabla_i\psi(x), T_i^{i}(x^{i})-x^{i}\rangle\,\mu_{n:i}(dx).
\end{align*}
Apply Lemma \ref{Nesterov inequality}, using that $T_i^j(x^j)-x^j=0$ for $j \neq i$, to get
\begin{equation*}
\int_{\R^k}\psi \,d\big(\mu_{n:i-1} - \mu_{n:i}\big) \ge \int_{\R^k} \langle \nabla_i\psi(x),T_i^i(x^i)-x^i\rangle \,\mu_{n:i}(dx) + \frac{1}{2L}\|\nabla \psi\circ T_i-\nabla \psi\|^2_{L^2(\mu_{n:i})}.
\end{equation*}
Add these inequalities to get
\begin{equation*}
\int_{\R^k}\psi \,d\big(\mu_{n:i-1} - \mu_{n:i}\big) + h(\mu_{n}^i) - h(\mu_{n+1}^i) \ge \frac{1}{2L}\|\nabla \psi\circ T_i-\nabla \psi\|^2_{L^2(\mu_{n:i})}.
\end{equation*}
Use the composite structure \eqref{eq:Hcomposite} to get
\begin{align*}
H(\mu_{n:i-1})-H(\mu_{n:i}) &= \int_{\R^k}\psi \,d\big(\mu_{n:i-1} - \mu_{n:i}\big) + \sum_{j=1}^d h(\mu^j_{n:i-1}) - h(\mu^j_{n:i}) \\
	&= \int_{\R^k}\psi \,d\big(\mu_{n:i-1} - \mu_{n:i}\big) + h(\mu_{n}^i) - h(\mu_{n+1}^i) \\
	&\ge \frac{1}{2L}\|\nabla \psi\circ T_i-\nabla \psi\|^2_{L^2(\mu_{n:i})}.
\end{align*}
Sum over $i=1,\ldots,d$ to complete the proof.
\end{proof}

We next establish a bound in the opposite direction, which is stated in terms of the $\W_2$-diameter of the CAVI iterates, called $R$ in Theorem \ref{main theorem}. At the end of the section we will elaborate on how to estimate this quantity.

\begin{proposition}
    \label{upper bound lipschitz case}
    Let $\mu_*$ be a minimizer, and let $R=\sup_{n \in \N}\W_2(\mu_n,\mu_*)$. Then $R < \infty$, and for each $n \ge 1$ we have
    \[\left(H(\mu_{n+1})-H(\mu_*)\right)^2\leq R^2d \sum_{i=1}^{d} \|\nabla\psi\circ T_i-\nabla\psi\|^2_{L^2(\mu_{n:i})},\]
    where again $T_i := T_{\mu_{n:i} \to \mu_{n:i-1}}$ is the Brenier map from $\mu_{n:i}$ to $\mu_{n:i-1}$.
\end{proposition}
\begin{proof}
Recall that $R$ is finite by Lemma \ref{moment bounds}.
Abbreviate $T=T_{\mu_{n+1}\to\mu_*}$, with coordinates $T(x)=(T^1(x^1),\ldots,T^d(x^d))$. 
Using geodesic convexity, in the form of Proposition \ref{characterization of convexity with gradient}, along with \eqref{eq:wgradH},
\begin{align*}
     H(\mu_*) &- H(\mu_{n+1}) \geq \int_{\R^{k}} \left \langle \wgrad H(\mu_{n+1},x), T(x)-x\right\rangle\,\mu_{n+1}(dx)\\
     &= \sum_{i=1}^d \int_{\R^{k}} \left\langle \nabla_i\psi(x) -\int_{\R^{k-k_i}}\nabla_i\psi(x^i,y^{-i})\,\mu^{-i}_{n:i}(dy^{-i}),T^i(x^i)-x^i\right\rangle\,\mu_{n+1}(dx)\\
     &=\sum_{i=1}^d \int_{\R^k} \left\langle\nabla_i \psi(x),T^i(x^i)-x^i\right\rangle\,(\mu_{n+1}-\mu_{n:i})(dx),
\end{align*}
with the last step using the identity of marginals $\mu_{n+1}^i=\mu^i_{n:i}$. The $i=d$ term in the sum vanishes because $\mu_{n+1}=\mu_{n:d}$.  Introduce a telescoping sum to get
\begin{align*}
H(\mu_{n+1})-H(\mu_*) &\le \sum_{i=1}^{d-1} \sum_{j=i+1}^{d} \left \vert\int_{\R^k} \left\langle\nabla_i \psi(x),T^i(x^i)-x^i\right\rangle\,(\mu_{n:j}-\mu_{n:j-1})(dx)\right\vert \\
	&= \sum_{i=1}^{d-1} \sum_{j=i+1}^{d} \left \vert\int_{\R^k} \left\langle\nabla_i \psi(x)-\nabla_i \psi \circ T_j(x),T^i(x^i)-x^i\right\rangle\,\mu_{n:j}(dx)\right\vert,
\end{align*}
where the second step uses the fact noted in the previous proof that $T^i_j(x^i)=x^i$ for all $j \neq i$. 
Interchange the summations and apply Cauchy-Schwarz to get
\begin{align*}
H(\mu_{n+1})-H(\mu_*) &\le \sum_{j=2}^d \sqrt{\sum_{i=1}^{j-1} \| \nabla_i \psi -\nabla_i \psi \circ T_j \|^2_{L^2(\mu_{n_:j})} }\sqrt{\sum_{i=1}^{j-1}  \int_{\R^k}|T^i(x^i)-x^i|^2\,\mu_{n:j}(dx) }.
\end{align*}
The first square root is bounded by the norm $\| \nabla \psi -\nabla \psi \circ T_j \|_{L^2(\mu_{n:j})}$. 
Noting that the $i$-th marginal of $\mu_{n:j}$ is $\mu_{n+1}^i$ for $i < j$, the second square root is
\begin{align*}
\sqrt{\sum_{i=1}^{j-1}  \W_2^2(\mu^i_{n+1},\mu_*^i) } \le \sqrt{\sum_{i=1}^{d}  \W_2^2(\mu^i_{n+1},\mu_*^i) } = \W_2(\mu_{n+1},\mu_*) \le R,
\end{align*}
where we used the tensorization identity stated in Lemma \ref{product structure of Brenier Maps}.
Hence,
\begin{align*}
H(\mu_{n+1})-H(\mu_*) &\le  R\sum_{j=1}^d \| \nabla \psi -\nabla \psi \circ T_j \|_{L^2(\mu_{n:j})}.
\end{align*}
Square both sides and apply Cauchy-Schwarz to complete the proof.
\end{proof}

Finally, we put the previous two bounds together:

\begin{proposition} 
\label{lipschitz gradient sharp rate lemma}
For $n \ge 1$,
    \[H(\mu_{n})-H(\mu_*) \leq C \frac{2LR^2d}{n},\]
    Where we define $R$ as in Proposition \ref{upper bound lipschitz case}, as well as $a = 1/(2LR^2d)$ and 
    \[C = \max\left\{2,\frac{4a(H(\mu_{1})-H(\mu_*))}{1+\sqrt{1+4a(H(\mu_{1})-H(\mu_*))}}\right\}.\]
\end{proposition}
\begin{proof}
    Let $Q_n = H(\mu_{n})-H(\mu_*)$. Combine Propositions \ref{upper bound lipschitz case} and \ref{lower bound lipschitz case} to get
    \[ aQ_{n+1}^2  \leq Q_{n}-Q_{n+1}.\] 
Solving the quadratic inequality, we obtain
    \[Q_{n+1} \leq \frac{-1 + \sqrt{1+4aQ_n}}{2a} =  \frac{2Q_n}{1+\sqrt{1+4aQ_n}}.\]
    In particular, $Q_2 \le C/2a$ if we define $C$ as above. 
    Now we proceed by induction. Assume $Q_n \leq C/an$ for some $n \ge 2$. By the quadratic inequality we get
    \begin{align*}
        Q_{n+1}&\le \frac{2C/an}{1+\sqrt{1+4C/n}}  =\frac{C}{a}\cdot \frac{2}{n+\sqrt{n^2+4Cn}} \leq \frac{C}{a} \cdot\frac{1}{n+1}.
    \end{align*}
    Where the last inequality follows because $C\geq 2$. This completes the proof.
\end{proof}

\begin{proof}[Proof of Theorem \ref{main theorem}(3)]
The bound of Proposition \ref{lipschitz gradient sharp rate lemma} is stronger than the one stated in Theorem \ref{main theorem}(3). Indeed, this is seen by applying the elementary inequalities $\max(x,y) \le x+y$, then $\frac{x}{1+\sqrt{1+x}}\leq \sqrt{x}$, and then $2\sqrt{xy} \le x+y$, for $x,y \ge 0$. 
\end{proof}

\subsection{Bounds on the diameter $R$} \label{se:Restimates}
Theorem \ref{main theorem}(3) is stated in terms of the somewhat mysterious constant $R$.
A common and usually innocuous assumption in convex optimization is that the level sets of the objective functions are bounded, but this is somewhat demanding in our infinite-dimensional setting. Nonetheless, the iterates of the CAVI algorithm remain in a $\W_2$-ball, and in this section we present some estimates of its size.
In full generality, it is exponential in $d$. Under stronger assumptions, we are able to get better estimates.
\begin{itemize}
\item For $r > 0$, a probability measure $\nu$ satisfies \emph{Talagrand's inequality} with constant $r$ if for each probability measure $\mu$ it holds
$\W_2^2(\mu,\nu)\leq 2r \h(\mu\,\|\,\nu)$ .
\item For $r > 0$, a probability measure $\nu$ is $r$-subgaussian if  $\int_{\R^k} e^{r |x|^2}\,\nu(dx)<\infty$.
\end{itemize}
Both concepts are standard, except perhaps for the way our constant $c$ enters in the latter, and the former implies the latter (for some different $r$). Talagrand's inequality is implied by a log-Sobolev inequality and is known to hold for any strongly log-concave measure; see \cite{gozlan2010transport} for examples and sufficient conditions as well as its connection with concentration of measure.

Recall in the following that, because $e^{-\psi}$ is a probability density and $\psi$ convex, the assumption (Q.1) holds automatically; see Remark \ref{re:convex->Q}. That is, there exist $\alpha \in \R$ and $\beta > 0$ such that $\psi(x) \ge \alpha + \beta|x|$ for all $x \in \R^k$. We use these constants in the following lemma.

\begin{lemma}
Let $\mu_*$ be a minimizer of  \eqref{MFVI-formal}, and define $R= \sup_{n \in \N}\W_2(\mu_n,\mu_*)$.
\begin{enumerate}[(1)]
\item It holds that
\begin{align*}
R &\leq 2 e^{\tfrac12(2d+1)H(\mu_1) -\tfrac12(d+1)\alpha}\left(\int_{\R^{k}} e^{-(\beta|x|+\alpha)\slash2}\,dx\right)^{d} \sqrt{\int_{\R^k} |x|^2 e^{-\beta\sum_i |x^i|}\,dx}.
\end{align*} 
\item If  $\rho$ satisfies Talagrand's inequality with constant $r$,  then
    \[R\leq  2\sqrt{2rH(\mu_1)} <\infty.\]
    \item If $\rho$ is $r$-subgaussian,  
    \[R\leq \frac{2}{\sqrt{r}} \left(H(\mu_1)+\log \int_{\R^k}e^{r |x|^2}\,\rho(dx)\right)^{1/2} <\infty.\]
\end{enumerate}
\end{lemma}
\begin{proof}{ \ }
\begin{enumerate}[(1)] 
\item  Start with the triangle inequality
    \begin{align*}
        \W_2(\mu_n,\mu_*) &\le 
        \sqrt{\int_{\R^{k}} |x|^2\,\mu_n(dx)}+\sqrt{\int_{\R^{k}} |x|^2\,\mu_*(dx)}. 
    \end{align*}
    Lemma \ref{moment bounds} bounds the second moment of $\mu_n$. Noting that $\mu_*$ is a fixed point of the CAVI algorithm by Lemma \ref{form of the iterates lemma}, Lemma \ref{moment bounds} applies to $\mu_*$ as well.
    \item If $\rho$ satisfies the Talagrand inequality, we instead obtain 
    \begin{align*}
        \W_2(\mu_n,\mu_*) &\leq \W_2(\mu_n,\rho)+\W_2(\rho,\mu_*) \leq \sqrt{2rH(\mu_n)}+\sqrt{2rH(\mu_*)} \leq 2\sqrt{2rH(\mu_1)} ,
    \end{align*}
    where we used the monotonicity of $H$ along CAVI iterates. Lemma \ref{form of the iterates lemma} shows that the right-hand side is finite.
    \item  If $\rho$ is $r$-subgaussian, we apply the triangle inequality as in (1), and then use the Gibbs variational principle \cite[Proposition 2.3(b)]{BudhirajaDupuis}:
    \[\int_{\R^k}|x|^2\,\mu_n(dx)\leq \frac{1}{r} \left(H(\mu_n)+\log \int_{\R^k}e^{r |x|^2}\,\rho(dx)\right),\]
    and then again $H(\mu_n) \le H(\mu_1)$. Argue similarly for $\mu_*$, using also $H(\mu_*) \le H(\mu_1)$.
\end{enumerate}
\vskip-.53cm
\end{proof}

\section{The strongly convex and Lipschitz gradient case}\label{strongly convex}

This section is devoted to the proof of Theorem \ref{main theorem}(3), and we impose the assumptions thereof throughout this section, most notably the strong $\lambda$-convexity of $\psi$ for $\lambda > 0$. 
As a result, $H$ and $\mu \mapsto \int_{\R^k}\psi\,d\mu$ are both $\lambda$-geodesically strongly convex functionals, as noted in Theorem \ref{th:geoconvex-examples}. 
The general strategy remains the same as in the previous section: obtain matching bounds and iterate, and we make use of the Wasserstein gradient formulas obtained in Lemma \ref{differentiability of entropy}. 
We adapt techniques developed in \cite{StronglyConvexCaseLi2018} for the Euclidean case.

\begin{proposition}
\label{lower bound strongly convex case}
For $n \ge 1$,
    \[H(\mu_n)-H(\mu_{n+1}) \geq \frac{\lambda}{2} \W_2^2(\mu_n,\mu_{n+1}).\]
\end{proposition}
\begin{proof}
Let $T_i=T_{\mu_{n:i}\to \mu_{n:i-1}}$ be the Brenier map from $\mu_{n:i}$ to $\mu_{n:i-1}$, for each $i \in [d]$.
    By strong convexity and Proposition \ref{characterization of convexity with gradient},
\begin{equation*}
H(\mu_{n:i-1})-H(\mu_{n:i}) \ge \int_{\R^k} \left \langle\wgrad H(\mu_{n:i},x) , T_{i}(x)-x\right \rangle\, \mu_{n:i}(dx) +\frac{\lambda}{2}\W_2^2(\mu_{n:i},\mu_{n:i-1}).
	\end{equation*}
    We next argue as in Proposition \ref{lower bound lipschitz case} that $H$ has nonempty subdifferential at $\mu_{n:i}$ and that the integral term vanishes.
Indeed, because $\mu_{n:i}$ and $\mu_{n:i-1}$ differ in only the $i$th marginal, we have $T_i^j(x^j)=x^j$ for all $j \neq i$. Applying  the identity $\mu_{n:i}^i=\mu_{n+1}^i$ and then Lemma \ref{differentiability of entropy} yields
\begin{align*}
\int_{\R^k} &\left \langle\wgrad H(\mu_{n:i},x) , T_{i}(x)-x\right \rangle\, \mu_{n:i}(dx) \\
	&= \int_{\R^k} \left \langle \nabla_i\psi(x) + \nabla_i \log \mu_{n+1}^i(x^i), T^i_i(x^i)-x^i\right \rangle\, \mu_{n:i}(dx) \\
	&= \int_{\R^k} \left \langle \nabla_i\psi(x) - \int_{\R^{k-k_i}} \nabla_i \psi(x^i,y^{-i})\,\mu_{n:i}^{-i}(dy^{-i}), T^i_i(x^i)-x^i\right \rangle\, \mu_{n:i}(dx) = 0.
\end{align*}
    Now sum over coordinates to obtain
    \begin{align*}
        H(\mu_n)-H(\mu_{n+1}) &= \sum_{i=1}^d \big(H(\mu_{n:i-1})-H(\mu_{n:i}\big) \\
        	&\geq \frac{\lambda}{2}\sum_{i=1}^{d}  \W_2^2(\mu_{n:i-1},\mu_{n:i}) = \frac{\lambda}{2}\sum_{i=1}^d \W_2^2(\mu_{n}^i,\mu_{n+1}^i).
    \end{align*}
    The proof is complete upon recalling the tensorization identity of Lemma \ref{product structure of Brenier Maps}.
\end{proof}

In the following, we let $\mu_*$ denote the minimizer for the MFVI problem \eqref{MFVI-formal}, which is unique because $H$ is strictly geodesically convex by Theorem \ref{th:geoconvex-examples}, and because $\P_2^{\otimes d}(\R^k)$ is a geodesically convex set as noted in Proposition \ref{geodesic convexity of product measures}.

\begin{proposition} 
\label{upper bound strongly convex case}
For $n \ge 1$,
\[H(\mu_{n+1})-H(\mu_*)\leq \frac{dL^2}{2\lambda}\W_2^2(\mu_n,\mu_{n+1}).\]
\end{proposition}
\begin{proof}
Let $T=T_{\mu_{n+1}\to\mu_*}$ denote the Brenier Map from $\mu_{n+1}$ to $\mu_*$, written in coordinates as $T(x)=(T^1(x^1),\ldots,T^d(x^d))$.
By strong convexity and Proposition \ref{characterization of convexity with gradient},
\begin{equation}
H(\mu_*)\geq H(\mu_{n+1})+\int_{\R^k} \left\langle\wgradp H(\mu_{n+1},x),T(x)-x\right\rangle\,\mu_{n+1}(dx)+\frac{\lambda}{2} \W_2^2(\mu_{n+1},\mu_*). \label{pf:ub1convex}
\end{equation}
The identity \eqref{eq:wgradH} yields
\begin{align*}
\int_{\R^k} &\left\langle\wgradp H(\mu_{n+1},x),T(x)-x\right\rangle\,\mu_{n+1}(dx) \\
	&=\sum_{i=1}^d \int_{\R^{k_i}} \left\langle \nabla_i\psi(x) - \int_{\R^{k-k_i}} \nabla_i \psi(x^i,y^{-i})\, \mu_{n:i}^{-i}(dy^{-i}), T^i(x^i)-x^i\right\rangle \, \mu_{n+1}(dx) \\
	&= \sum_{i=1}^d \int_{\R^{k_i}}  \left\langle\int_{\R^{k-k_i}} \nabla_i \psi (x^i,y^{-i})\,[\mu_{n+1}^{-i}-\mu_{n:i}^{-i}](dy^{-i}), T^i(x^i)-x^i\right\rangle \,\mu_{n+1}^{i}(dx^i).
\end{align*}
Optimality of the transport map $T$ yields
\[
\frac{\lambda}{2} \W_2^2(\mu_{n+1},\mu_*) = \frac{\lambda}{2} \sum_{i=1}^d \int_{\R^{k_i}}|T^i(x^i)-x^i|^2\,\mu_{n+1}^{i}(dx^i).
\]
Plug the two preceding identities into \eqref{pf:ub1convex} and complete the square to get
\begin{align*}
        H(\mu_{n+1})-H(\mu_*)&\leq \frac{1}{2\lambda}\sum_{i=1}^d \int_{\R^{k_i}} \left\vert\int_{\R^{k-k_i}} \nabla_i \psi (x^i,y^{-i})\,[\mu_{n+1}^{-i}-\mu_{n:i}^{-i}](dy^{-i})\right\vert^2\,\mu^i_{n+1}(dx^i). 
\end{align*}
The function $y^{-i} \mapsto \nabla_i\psi(x^i,y^{-i})$ is $L$-Lipschitz. We may thus apply Kantorovich duality \cite[Theorem 1.14]{villani2021topics} and Jensen's inequality (to bound the 1-Wasserstein distance by the 2-Wasserstein distance) to get
\begin{equation*}
\left\vert\int_{\R^{k-k_i}} \nabla_i \psi (x^i,y^{-i})\,[\mu_{n+1}^{-i}-\mu_{n:i}^{-i}](dy^{-i})\right\vert \le L \W_2(\mu_{n+1}^{-i},\mu_{n:i}^{-i}) \le L \W_2(\mu_{n+1},\mu_{n:i}).
\end{equation*}
Hence, recalling the tensorization identity of Lemma \ref{product structure of Brenier Maps},
\begin{align*}
H(\mu_{n+1})-H(\mu_*) &\leq \frac{L^2}{2\lambda} \sum_{i=1}^d \W_2^2(\mu_{n+1},\mu_{n:i}) = \frac{L^2}{2\lambda}\sum_{i=1}^d\sum_{j =i+1}^d \W_2^2(\mu_n^j,\mu_{n+1}^j) \\
    &\leq \frac{L^2}{2\lambda}\sum_{i=1}^d\sum_{j=1}^d\W_2^2(\mu_n^j,\mu_{n+1}^j) \\
    &=\frac{dL^2}{2\lambda}\W_2^2(\mu_n,\mu_{n+1}). \qedhere
\end{align*}
\end{proof}

The first claimed inequality of Theorem \ref{main theorem}(4) will be obtained as a consequence of the following strong convexity inequality  which may be of independent interest:

\begin{proposition}
    \label{pr: strongly convex functional inequality}
For every $\nu \in \mathcal{P}_{2,ac}^{\otimes d} (\mathbb{R}^k)$, we have
    \[H(\nu)-H(\mu_*)\geq \frac{\lambda}{2}\W_2^2(\nu,\mu_*).\]
\end{proposition}
\begin{proof}
As a fixed point of the CAVI iteration, $\mu_*$ satisfies the mean field equations as in Proposition \ref{mean field equation}, and in particular
\begin{equation*}
\nabla_i\log \mu_*^i (x^i) = - \int_{\R^{k-k_i}} \nabla_i \psi(x^i,y^{-i})\,\mu_*^{-i}(dy^{-i}).
\end{equation*}
Lemma \ref{differentiability of entropy} applies to $\mu_*$ to show that $H$ admits nonempty subdifferential at $\mu_*$. Recall from Lemma \ref{product structure of Brenier Maps} the product form $T_{\mu_*\to\nu}(x)=(T_{\mu_*\to\nu}^1(x^1),\ldots,T_{\mu_*\to\nu}^n(x^n))$ for the Brenier map between product measures.
    Use Proposition \ref{characterization of convexity with gradient} followed by Theorem \ref{differentiability of relative internal energy} to get
\begin{align*}
    H&(\nu)-H(\mu_*) - \frac{\lambda}{2}\W_2^2(\nu,\mu_*)\geq \int_{\R^{k}} \left\langle \wgradp H(\mu_*,x),T_{\mu_* \to \nu}(x)-x\right\rangle\,\mu_*(dx) \\
         &=\sum_{i=1}^d \int_{\R^k}\left\langle \nabla_i \psi(x)+\nabla_i\log \mu_*^i (x^i),T^i_{\mu_*^i \to \nu^i}(x^i)-x^i\right\rangle\,\mu_*(dx) \\
         &=\sum_{i=1}^d \int_{\R^{k_i}}\!\!\left\langle\int_{\R^{k-k_i}} \nabla_i \psi(x^i,y^{-i})\,\mu_*^{-i}(dy^{-i})+\nabla_i\log \mu_*^i (x^i),T^i_{\mu_*^i \to \nu^i}(x^i)-x^i\right\rangle \mu_*^i(dx^i) \\
         &= 0. \qedhere
\end{align*}
\end{proof}

\begin{proof}[Proof of Theorem \ref{main theorem}(4)]
Combine Proposition \ref{lower bound strongly convex case} and \ref{upper bound strongly convex case} to get
    \begin{align*}
        H(\mu_{n})-H(\mu_*)&=H(\mu_n)-H(\mu_{n+1})+H(\mu_{n+1})-H(\mu_*) \\
        &\geq\frac{\lambda}{2} \W_2^2(\mu_n,\mu_{n+1}) +H(\mu_{n+1})-H(\mu_*)\\
        &\geq \frac{\lambda^2}{L^2 d} (H(\mu_{n+1})-H(\mu_*)) + H(\mu_{n+1})-H(\mu_*).
    \end{align*}
    Rearrange and iterate to find
\begin{align*}
H(\mu_{n+1})-H(\mu_*) &\leq \Big(1+\frac{\lambda^2}{L^2 d}\Big)^{-1}(H(\mu_{n})-H(\mu_*))\leq \Big(1+\frac{\lambda^2}{L^2 d}\Big)^{-n}(H(\mu_{1})-H(\mu_*)).
\end{align*}
To find the bound on $\W_2^2(\mu_n,\mu_*)$ apply Proposition (\ref{pr: strongly convex functional inequality}).
\end{proof}
\appendix

\section{Gaussian Case}
\label{se:gaussian}

Here we prove the dimension-free exponential convergence in the Gaussian case. Recall here that there are $d=k$ blocks of dimension 1, and $\rho \sim \mathcal{N}(m,A)$ for a given vector $m$ and a positive definite matrix $A$.

\begin{proof}[Proof of Proposition \ref{gaussian case}]
Let $m_n$ be the mean vector of $\mu_n$, and let $a_{ij}$ be the entries of $A$. Using the formula \eqref{def:update-explicit} for iterates,
\begin{align*}
    \mu_{n+1}^1(x^1) &\propto \exp\bigg(-\frac{1}{2}\int (x-m)^\top A(x-m)\,\mu^{-1}_n(dx^{-1})\bigg)\\
    &\propto \exp\bigg(-\frac{1}{2}a_{11}(x^1-m^1)^2-(x^1-m^1)\sum_{j > 1} a_{1j}(m^j_n-m^j) \bigg). 
\end{align*}
In turn, this implies that $\mu^1_{n+1} \sim \mathcal{N}(m_{n+1}^1,1/a_{11})$ where 
\[m_{n+1}^1 =m^1-\frac{1}{a_{11}}\sum_{j > 1}a_{1j}(m_n^j-m^j).\]
By induction, for $n \ge 0$ and $i \in [d]$, each marginal $\mu_{n+1}^i$ is Gaussian with variance $1/a_{ii}$, and with means obeying the following update rule:
\begin{equation}
\label{gaussian update rule}
    m_{n+1}^i = m^i-\frac{1}{a_{ii}}\left(\sum_{j < i }a_{ij}(m_{n+1}^j-m^j)+\sum_{j < i }a_{ij}(m_{n}^j-m^j)\right).
\end{equation}
In particular, the unique fixed point of this iteration is $m_n=m$, and we confirm the well known result that $\mu_* \sim \mathcal{N}(0,(A \odot I)^{-1})$.
A well known formula for Gaussians yields, for $n \ge 1$,
\[H(\mu_n\,\|\,\rho)-H(\mu_*\,\|\,\rho)= \frac{1}{2}(m_n-m)^\top A(m_n-m)=\psi(m_n).\]
We now prove that the update rule \eqref{gaussian update rule} coincides with the update rule obeyed by the BCD algorithm applied to $\psi(x)$. Indeed, if we set $\nabla_i\psi(x)=0$ and solve for $x^i$ in terms of the other coordinates $(x^j)_{j \neq i}$, we obtain
\begin{align*}
    x^i=m^i-\frac{1}{a_{ii}}\sum_{j \neq i} (x^j-m^j)a_{i,j}.
\end{align*}
Hence, applying \cite[Theorem 4]{StronglyConvexCaseLi2018}, and noting that $\psi(m)=0$,
\begin{align*}
\psi(m_n) &\leq \left(1-\frac{\lambda^2}{\lambda^2+64(L-\lambda)^2 \log^2(3)}\right)^{n-1}\psi(m_1). \qedhere
\end{align*}
\end{proof}

\section{Norm compactness of CAVI densities} \label{se:compactness}

In this section we continue to use the notation for CAVI iterates introduced in Section \ref{CAVI ALGORITHM SECTION}, and we prove the second claim of Remark \ref{re:convergencemode}.

\begin{proposition} \label{pr:normcompact}
Let $q \in [1,\infty)$.
Under the assumptions of Theorem \ref{main theorem}, the sequence of densities of CAVI iterates $(\mu_n)_{n \ge 1}$ is precompact in $L^q(\R^k)$.
\end{proposition}
\begin{proof}
We first show that
\begin{equation}
\sup_{n \ge 1}\int_{\R^k}\big( |\mu_n(x)|^q + |\nabla \mu_n(x)|^q\big)\,dx < \infty. \label{Sobolev1}
\end{equation}
Let $Z_{n+1}$ be the integration constant of $\mu_{n+1}$ hidden in the definition of the iterates \eqref{eq:iterates2}. It was shown in the proof of Lemma \ref{moment bounds} that $\inf_{n \ge 1} Z_n > 0$. Using $\psi(x) \ge \alpha + \beta|x^i|$ for each $i$, it follows that
\begin{align}
 |\mu_{n+1}(x)|^q &= \frac{1}{Z_{n+1}^q} \exp\left\{ -q\sum_{i=1}^d \int_{\R^{k-k_i}} \psi(x^i,y^{-i})\,\mu_{n:i-1}^{-i}(dy^{-i})\right\}  \nonumber \\
	&\le \frac{1}{Z_{n+1}^q} \exp\left\{ -q\alpha - q\beta\sum_{i=1}^d|x^i|\right\} \label{expbound-q}
\end{align}
has finite integral, uniformly in $n \ge 0$. 
Apply Lemma \ref{differentiability of entropy} to get
note also that
\begin{align*}
 |\nabla \mu_{n+1}(x)|^{q} &= \int_{\R^k} |\nabla \log \mu_{n+1}(x)|^{q}|\mu_{n+1}(x)|^{q} \,dx \\
	&= \frac{1}{Z_{n+1}^{q}}\int_{\R^k} \bigg( \sum_{i=1}^d \bigg|\int_{\R^{k-k_i}} \nabla_i \psi(x^i,y^{-i})\,\mu_{n:i}^{-i}(dy^{-i})\bigg|^2 \bigg)^{q/2}|\mu_{n+1}(x)|^{2q}.
\end{align*}
Recall that $|\nabla \psi|$ has polynomial growth, and $\mu_n$ has finite moments of all orders bounded uniformly in $n$ by Lemma \ref{moment bounds}. Combining these facts with \eqref{expbound-q}, we easily deduce \eqref{Sobolev1}.

With \eqref{Sobolev1} established, we may apply the Rellich-Kondrachov theorem \cite[Theorem 8.9]{lieb2001analysis} to deduce that $(\mu_n|_B)_{n \ge 1}$ is precompact in $L^q(B)$ for any $q \ge 1$ and for any ball $B \subset \R^k$. By a diagonal argument, for any subsequence of $(\mu_n)_{n \ge 1}$ we may thus extract a further subsequence which converges to some $\mu_*$ in $L^q(B)$ for every centered ball $B$ of integer radius (say). Let $B_r\subset \R^k$ be the centered ball of radius $r$. The estimate \eqref{expbound-q} clearly shows that
\begin{align*}
\lim_{r\to\infty} \sup_{n \ge 1}\int_{B_r^c} |\mu_{n}(x)|^q\,dx = 0, 
\end{align*}
and it follows readily that $(\mu_n)_{n \ge 1}$ in fact converges in $L^q(\R^k)$ to the same $\mu_*$.
\end{proof}

\bibliographystyle{abbrv}
\bibliography{Biblio}

\end{document}